\documentclass{article}

\usepackage{lipsum}
\usepackage{titletoc}
\usepackage{microtype}
\usepackage{graphicx}
\usepackage{booktabs} 
\usepackage{xcolor}         
\usepackage{wrapfig}

\usepackage[pagebackref]{hyperref}  

\hypersetup{colorlinks=true,linkcolor=red!70!black,linktocpage=false,citebordercolor=blue!70!black,citecolor=blue!70!black,anchorcolor=blue!70!black}

\usepackage[ruled,vlined]{algorithm2e}

\usepackage[accepted]{icml2024}

\usepackage{amsmath,amsfonts,bm}

\def\1{\bm{1}}

\def\beps{\boldsymbol{\epsilon}}

\def\vw{{\bm{w}}}

\DeclareMathAlphabet{\mathsfit}{\encodingdefault}{\sfdefault}{m}{sl}
\SetMathAlphabet{\mathsfit}{bold}{\encodingdefault}{\sfdefault}{bx}{n}

\newcommand{\be}{\boldsymbol{e}}

\newcommand{\bk}{\boldsymbol{k}}

\newcommand{\bu}{\boldsymbol{u}}
\newcommand{\bv}{\boldsymbol{v}}
\newcommand{\bw}{\boldsymbol{w}}
\newcommand{\bx}{\boldsymbol{x}}

\newcommand{\bz}{\boldsymbol{z}}

\newcommand{\rank}{{\rm{rank}}}

\newcommand{\brho}{\bm{\rho}}

\newcommand{\btheta}{\boldsymbol{\theta}}

\newcommand{\cI}{\mathcal{I}}

\newcommand{\cL}{\mathcal{L}}

\newcommand{\cO}{\mathcal{O}}
\newcommand{\cP}{\mathcal{P}}

\newcommand{\cS}{\mathcal{S}}

\newcommand{\bbA}{\mathbb{A}}
\newcommand{\bbB}{\mathbb{B}}
\newcommand{\bbC}{\mathbb{C}}

\newcommand{\bbR}{\mathbb{R}}
\newcommand{\bbS}{\mathbb{S}}

\newcommand{\bzero}{\mathbf{0}}

\newcommand{\pll}{\kern 0.56em/\kern -0.8em /\kern 0.56em}

\newcommand{\norm}[1]{\ensuremath{\left\| #1 \right\|}}
\newcommand{\bracket}[1]{\ensuremath{\left( #1 \right)}}

\newcommand{\Proj}{{\rm Proj}}
\newcommand{\sgn}{{\rm sgn}}

\newcommand{\<}{\left\langle}
\renewcommand{\>}{\right\rangle}

\usepackage{graphicx,subfig}
\usepackage{url}
\usepackage{multirow}
\usepackage{amsmath}
\usepackage{amssymb}
\usepackage{mathtools}
\usepackage{amsthm}
\usepackage{enumitem}
\usepackage{stfloats}
\usepackage{makecell}

\usepackage[capitalize,noabbrev]{cleveref}

\theoremstyle{plain}
\newtheorem{theorem}{Theorem}[section]
\newtheorem{proposition}[theorem]{Proposition}
\newtheorem{lemma}[theorem]{Lemma}

\theoremstyle{definition}
\newtheorem{definition}[theorem]{Definition}
\newtheorem{assumption}[theorem]{Assumption}
\theoremstyle{remark}
\newtheorem{remark}[theorem]{Remark}
\newtheorem{dataset}{Dataset}

\allowdisplaybreaks[4]

\usepackage[textsize=tiny]{todonotes}

\icmltitlerunning{Achieving Margin Maximization Exponentially Fast via Progressive Norm Rescaling}

\begin{document}

\twocolumn[
\icmltitle{Achieving Margin Maximization Exponentially Fast \\ via Progressive Norm Rescaling}




\begin{icmlauthorlist}
\icmlauthor{Mingze Wang}{math}
\icmlauthor{Zeping Min}{math}
\icmlauthor{Lei Wu}{math,cmlr}
\end{icmlauthorlist}

\icmlaffiliation{math}{School of Mathematical Sciences, Peking University, Beijing, China}
\icmlaffiliation{cmlr}{Center for Machine Learning Research, Peking University, Beijing, China}

\icmlcorrespondingauthor{Mingze Wang}{mingzewang@stu.pku.edu.cn}
\icmlcorrespondingauthor{Lei Wu}{leiwu@math.pku.edu.cn}

\icmlkeywords{Machine Learning, ICML}

\vskip 0.3in
]



\printAffiliationsAndNotice{}  

\begin{abstract}
In this work, we investigate the margin-maximization bias exhibited by gradient-based algorithms in classifying linearly separable data. We present an in-depth analysis of the specific properties of the velocity field associated with (normalized) gradients, focusing on their role in margin maximization. Inspired by this analysis, we propose a novel algorithm called Progressive Rescaling Gradient Descent (PRGD) and show that PRGD can maximize the margin at an {\em exponential rate}. This stands in stark contrast to all existing algorithms, which maximize the margin at a slow {\em polynomial rate}. Specifically, we identify mild conditions on data distribution under which existing algorithms such as gradient descent (GD) and normalized gradient descent (NGD) {\em provably fail} in maximizing the margin efficiently. To validate our theoretical findings, we present both synthetic and real-world experiments. Notably, PRGD also shows promise in enhancing the generalization performance when applied to linearly non-separable datasets and deep neural networks.
\end{abstract}

\section{Introduction}

In modern machine learning, models are often over-parameterized in the sense that they can easily interpolate all training data, giving rise to a loss landscape with many global minima. Although all these minima yield zero training loss, their generalization ability can vary significantly. Intriguingly, it is often observed that Stochastic Gradient Descent (SGD) and its variants consistently converge to solutions with favorable generalization properties even without needing any explicit regularization~\citep{neyshabur2014search,zhang2017understanding}. This phenomenon implies that the ``implicit bias'' of SGD plays a crucial role in ensuring the efficacy of deep learning; therefore, revealing the underlying mechanism is of paramount importance.

\citet{soudry2018implicit} investigated implicit bias of GD for classifying linearly separable data with linear models. They showed that GD trained with exponentially-tailed loss functions can implicitly maximize the $\ell_2$-margin during its convergence process, ultimately locating a max-margin solution.
This discovery offers valuable insights into the superior generalization performance often observed with GD, as larger margins are generally associated with better generalization~\citep{boser1992training,bartlett2017spectrally}.
However, the  rate at which GD maximizes the margin has been shown to be merely  $\mathcal{O}(1/\log t)$ \citep{soudry2018implicit}. This naturally leads to the question: can we design a better gradient-based algorithm to  accelerate the margin maximization.
In the pursuit of this, \citet{nacson2019convergence,ji2021characterizing} has demonstrated that employing  GD with aggressively loss-scaled step sizes can achieve  polynomial rates in  margin maximization. Notably,  \citet{ji2021characterizing} specifically established that the rate of NGD is $\mathcal{O}(1/t)$. 
Building on this, \citet{ji2021fast} further introduced a momentum-based gradient method by applying Nesterov acceleration to the dual formulation of this problem, which achieves a remarkable margin-maximization rate of $\mathcal{O}(\log t/t^2)$ and  \citet{wang2022accelerated} further improved it to $\cO(1/t^2)$, currently standing as the state-of-the-art algorithm for this problem.

\textbf{Our Contributions.} 
In this paper, we begin by introducing a toy dataset to elucidate the causes of inefficiency in GD/NGD and to clarify the underlying intuition for accelerating margin maximization. Subsequently, we demonstrate that these insights are applicable to a broader range of scenarios. 
\begin{itemize}[leftmargin=2em]
    \item We reveal that the rate of directional convergence and  margin maximization is governed by  the centripetal velocity--the component orthogonal to the max-margin direction. We show that under mild conditions on data distribution,  NGD and GD will inevitably be trapped in a region where the centripetal velocity is diminished, thereby explaining the inefficiency of GD/NGD. Specifically, we establish that the aforementioned margin-maximization rates: $\mathcal{O}(1/\log t)$ for GD and  $\mathcal{O}(1/t)$ for NGD also serve as {\em lower bounds}.

    \item Based on the above observations, we propose to speed up the margin maximization by maintaining a non-degenerate centripetal velocity. We show that there exists a favorable region, where the centripetal velocity is uniformly lower-bounded and moreover, we can  reposition parameters into this region via a simple {\em norm rescaling}. Leveraging these properties,  we introduce an algorithm called Progressive Rescaling Gradient Descent (PRGD). Notably, we prove that PRGD can achieve both directional convergence and margin maximization at an {\em exponential} rate $\mathcal{O}(e^{-\Omega(t)})$. This stands in stark contrast to all existing algorithms, which maximize the margin at a slow polynomial rate. 
    \item Lastly, we validate our theoretical findings through both synthetic and real-world experiments.
    In particular, when applying PRGD to linearly non-separable datasets and homogenized deep neural networks—beyond the scope of our theory—we still observe consistent test performance improvements.
\end{itemize}
In Table~\ref{table: comparison tight}, we summarize our main theoretical results and compare them with existing ones. 

\begin{table*}[!bp]
\caption{ Comparison of the directional convergence and margin maximization rates of different algorithms under Assumption~\ref{ass: linearly separable},~\ref{ass: non-degenerate data}, and $\gamma^\star\bw^\star\ne\frac{1}{|\mathcal{I}|}\sum_{i\in\mathcal{I}}\bx_i y_i$. In this table,  $\bw^\star$ and $\gamma^\star$ denote the $\ell_2$ max-margin solution and the corresponding margin, respectively and $\bw(t)$ denotes the solution at the $t$-th step. 
}
\begin{center}
\begin{tabular}{c|c}
    \hline\hline
      Algorithm & 
      Error of Direction ${\rm e}(t)=\norm{\hat{\bw}(t)-\bw^\star}$
     \\\hline   
 GD & ${\rm e}(t)=\mathcal{O}(1/\log t)$ \citep{soudry2018implicit}; ${\color{red!70!black} {\rm e}(t)=\Theta\left(1/\log t\right)}$ ({\bf Thm}~\ref{thm: GD NGD main result}) \\ 
      NGD& ${\rm e}(t)=\mathcal{O}(1/ t)$ \citep{ji2021characterizing}; ${\color{red!70!black} {\rm e}(t_k)= \Theta(1/t_k)}$ ({\bf Thm}~\ref{thm: GD NGD main result})
      \\\hline 
      {\bf\color{red!70!black} PRGD} (Ours) & ${\color{red!70!black} {\rm e}(t)=e^{-\Omega(t)}}$ ({\bf Thm}~\ref{thm: PRGD main thm})
     \\ \hline\hline
\end{tabular}
\begin{tabular}{c|c}
 \multicolumn{2}{c}{}
    \\\hline\hline
    Algorithm & 
        Error of Margin ${\rm r}(t)=\gamma^\star-\gamma(\bw(t))$ 
         \\\hline   
     GD & ${\rm r}(t)=\mathcal{O}(1/\log t)$ \citep{soudry2018implicit}; ${\color{red!70!black} {\rm r}(t)=\Omega\left(1/\log^2 t\right)}$ ({\bf Thm}~\ref{thm: GD NGD main result}) \\ 
          NGD& ${\rm r}(t)=\mathcal{O}(1/ t)$ \citep{ji2021characterizing}; ${\color{red!70!black} {\rm r}(t_k)=\Omega(1/t_k^2)}$ ({\bf Thm}~\ref{thm: GD NGD main result})\\ 
         Nesterov Acceleration & ${\rm r}(t)=\tilde{\cO}(1/t^2)$ \citep{ji2021fast,wang2022accelerated} \\\hline 
          {\bf\color{red!70!black} PRGD} (Ours) & ${\color{red!70!black}{\rm r}(t)=e^{-\Omega(t)}}$ ({\bf Thm}~\ref{thm: PRGD main thm})
         \\ \hline\hline
    \end{tabular}
\end{center}
\label{table: comparison tight}
\end{table*}

\section{Related Work}


Unraveling the implicit bias of optimization algorithms has become a fundamental problem in theoretical deep learning and has garnered extensive attention recently.



\paragraph*{Margin Maximization in Gradient-based Algorithms.} 
The tendency of GD to favor max-margin solutions when trained with exponentially-tailed loss functions was first identified in the seminal work by \citet{soudry2018implicit}. Beyond aforementioned studies, 
  \citet{nacson2019stochastic} explored this bias for SGD and  \citet{gunasekar2018characterizing,wang2021momentum,sun2022mirror,wang2023faster} extended the analysis to various other optimization algorithms. Notably, \citet{ji2018risk} considered the  situation where dataset are linearly non-separable, providing  insights into the robustness of these findings in more complex settings. More recently, \citet{ji2020gradient} investigated the impact of the tail behavior of loss functions and 
\citet{wu2023implicit} analyzed the impact of edge of stability, a phenomenon previously observed by \citet{wu2018sgd,Jastrzebski2020The,cohen2021gradient}. 

Additionally, the margin-maximization analysis  has also been extended to nonlinear models. This includes studies by \citet{ji2018gradient,gunasekar2018implicit} on deep linear networks, as well as research by \citet{chizat2020implicit} on wide two-layer ReLU networks. Notably, \citet{nacson2019lexicographic,lyu2019gradient,ji2020directional} demonstrated that for general homogeneous networks, Gradient Flow (GF) and GD converge to solutions corresponding the KKT point of the max-margin problem. \citet{kunin2022asymmetric} has recently extended this analysis to quasi-homogeneous networks. Moreover, for two-layer (leaky-)ReLU neural networks, \citet{lyu2021gradient,vardi2022margin,wang2023understanding} studied whether the convergent KKT point of GF is a global optimum of the max-margin problem.

\textbf{Other Implicit Biases.} There are many other attempts to explain the implicit bias of deep learning algorithms \citep{vardi2023implicit}. Among them, the most popular one is the {\em flat minima hypothesis}: SGD favors flat minima \citep{keskar2016large} and flat minima generalize well \citep{hochreiter1997flat,jiang2019fantastic}. Recent studies \citep{wu2018sgd,ma2021linear,wu2022does} provided explanations for why SGD tends to select flat minima from a dynamical stability perspective. 
Moreover, \citet{blanc2020implicit,li2021happens,lyu2022understanding,ma2022beyond} offered in-depth characterizations of the dynamical process of  SGD in reducing the sharpness near the global minima manifold. Additionally, beyond empirical observations,  recent studies \citep{ma2021linear,mulayoff2021implicit,gatmiry2023inductive,wu2023implicitstability} provided theoretical evidence for the superior generalization performance of flat minima.
Besides,~\citet{woodworth2020kernel,pesme2021implicit,nacson2022implicit,pesme2023saddle,even2023s} investigated the implicit bias on linear diagonal networks, such as how the initialization scale and the step size affect the selection bias of GF, GD, and SGD in different regimes.
Additionally, various studies explored how  other training components impact implicit bias, such as normalization~\citep{wu2020implicit,li2020reconciling,lyu2022understanding,dai2023crucial}, re-parametrization~\citep{li2022implicit}, weight decay~\citep{andriushchenko2023we}, cyclic learning rate~\citep{wang2023noise}, and sharpness-aware minimization~\citep{wen2023how,wen2023sharpness,long2023sharpness}.

\section{Preliminaries}
\label{section: Preliminaries}

\textbf{Notation.} We use bold letters for vectors and lowercase letters for scalars, e.g. $\bx=(x_1,\cdots,x_d)^\top\in\mathbb{R}^d$. 
We use $\left<\cdot,\cdot\right>$ for the standard Euclidean inner product between
two vectors, and $\left\|\cdot\right\|$ for the $\ell_2$ norm of a vector or the spectral norm of a matrix.  Let $\hat{\bw}=\bw/\|\bw\|$ the normalized vector.
We use standard big-O notations $\cO,\Omega,\Theta$ to hide absolute positive constants, and use  $\tilde{\cO},\tilde{\Omega},\tilde{\Theta}$ to further hide logarithmic constants.
For any positive integer $n$, let $[n]=\{1,\cdots,n\}$. 


{\bf Problem Setup.} We consider the problem of binary classification  with a linear decision function $\bx\mapsto\<\bw,\bx\>$. Let $\cS=\{(\bx_1,y_1),\cdots,(\bx_n,y_n)\}_{i=1}^n$ with $\bx_i\in \bbR^{d}$ and $y_i\in \{\pm1\}$ for any $i\in [n]$ be the training set. 
Without loss of generality, we assume  $\norm{\bx_i}\leq1,\forall i\in[n]$. Throughout this paper, we assume $\cS$ to be linearly
separable:
\begin{assumption}[linear separability]\label{ass: linearly separable}
There exists a $\bw\in\bbS^{d-1}$ such that $\min\limits_{i\in [n]}y_i\<\bw,\bx_i\>>0$.
\end{assumption}

Under this assumption, the solutions that  classify all training data correctly may  not be unique.  Among them, {\bf the max-margin solution} is often favorable due to its superior generalization ability as suggested by the theory of support vector machine \citep{vapnik1999nature}. For any $\bw\in\mathbb{R}^d$,  the normalized $\ell_2$-margin is defined by $\gamma(\bw):=\min_{i\in[n]}y_i\<\hat{\bw},\bx_i\>$. The max-margin  solution and the corresponding max margin are defined by 
\begin{equation}
\bw^\star:=\mathop{\arg\max}\limits_{\bw \in \mathbb{S}^{d-1}}\gamma(\bw),\quad
\gamma^\star:=\gamma(\bw^\star).
\end{equation}
For the margin function, we have  the following important properties:
\begin{itemize}[leftmargin=2em]
    \item {\bf Homogeneity.} $\gamma(c\bw)=\gamma(\bw)$ for any $c>0$ and $\bw\in\bbR^d$. 
    \item {\bf Directional Convergence.}  Under Assumption~\ref{ass: linearly separable},  $\gamma^\star-\gamma(\bw)\leq\norm{\hat{\bw}-\bw^\star}$ (Lemma~\ref{lemma: Margin error and Directional error}).
\end{itemize}
Therefore, instead of directly inspecting the margin maximization,  we can focus on  the analysis of directional convergence, which is often must easier. 

In this paper, we are interested in algorithms that   minimize  the following objective
\begin{equation}\label{pro: logistic regression}
\cL(\bw)=\frac{1}{n}\sum_{i=1}^n\ell\left(y_i\<\bw,\bx_i\>\right).
\end{equation}
where $\ell:\mathbb{R}\mapsto\mathbb{R}_{\geq 0}$ is a loss function. We assume $\ell(z)=e^{-z}$ for simplicity and the extension to general loss functions with exponential-decay tails such as logistic loss $\ell(z)=\log(1+e^{-z})$ are straightforward~\citep{soudry2018implicit,nacson2019convergence}.



Consider to solve the optimization problem \eqref{pro: logistic regression} with GD:
\begin{equation}\label{equ: GD}
        \textbf{GD:}\quad\ 
        \bw(t+1)=\bw(t)-\eta\nabla\cL(\bw(t)).
\end{equation}
 \citet{soudry2018implicit} showed under Assumption~\ref{ass: linearly separable}, GD~\eqref{equ: GD} with $\eta\leq1$ converges in direction to the max-margin solution $\bw^\star$ despite the non-uniqueness of solutions. However, this occurs at a slow logarithmic rate $\gamma^\star-\gamma(\bw(t))=\cO(1/\log t)$. To accelerate the convergence, \citet{nacson2019convergence,ji2021characterizing} proposed the following Normalized Gradient Descent (NGD):
 \begin{equation}\label{equ: NGD}
        \textbf{NGD:}\quad\ 
        \bw(t+1)=\bw(t)-\eta\frac{\nabla\cL(\bw(t))}{\cL(\bw(t))},
\end{equation}
and \citet{ji2021characterizing} show that for NGD with $\eta\leq 1$, the margin maximization is much faster: $\gamma^\star-\gamma(\bw(t))=\cO(1/t)$.

\section{Motivations and the Algorithm}
\label{section: methodology}



In this section, we propose a toy problem to showcase why NGD is slow in maximizing the margin and why our proposed algorithm can accelerate it significantly.

\begin{dataset}\label{ass: dataset: 3data not intersect}  $\cS=\{(\bx_1, y_1),(\bx_2, y_2),(\bx_3, y_3)\}$ where $\bx_1=({\gamma^\star},\sqrt{1-{\gamma^\star}^2})^\top$, $y_1=1$, $\bx_2=({\gamma^\star},-\sqrt{1-{\gamma^\star}^2})^\top$, $y_2=1$, $\bx_3=(-\gamma^\star,-\sqrt{1-{\gamma^\star}^2})^\top$, $y_3=-1$, and $\gamma^\star>0$.
\end{dataset}
For this particular dataset, the max-margin solution is $\bw^\star=\be_1=(1,0)^\top$ and $\gamma^\star$ represents the associate  margin. A visualization of this dataset can be found in Figure~\ref{fig: 3 data vis}. 
\begin{figure*}[!ht]
    \hspace{-.3cm}
    \centering
    \subfloat[\small Dataset~\ref{ass: dataset: 3data not intersect}.]{\label{fig: 3 data vis}
    \includegraphics[width=2.4cm]{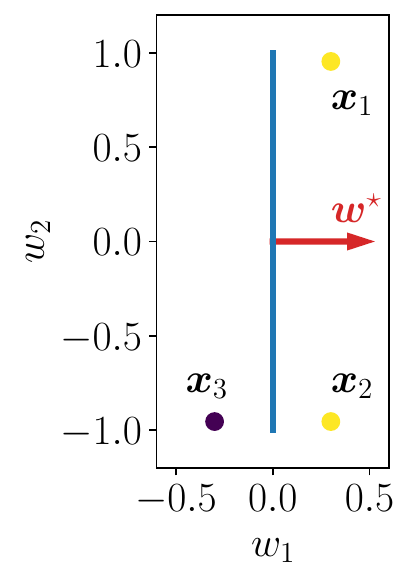}}
    \hspace{-.0cm}
    \subfloat[\small The vector field and the trajectories of NGD and PRGD for Dataset~\ref{ass: dataset: 3data not intersect}.]{\label{fig: 3 data vector field}
    \includegraphics[width=13.25cm]{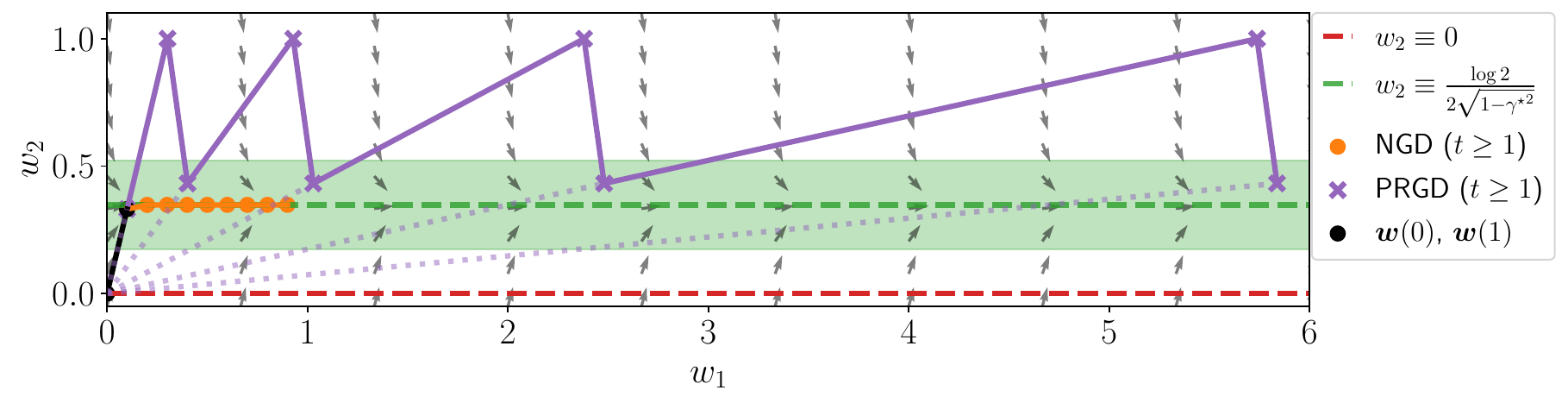}}
    \caption{\small (a) A visualization of Dataset~\ref{ass: dataset: 3data not intersect} where $\bw^\star$ is the max-margin solution.  (b) The vector field and the trajectories of NGD and PRGD for Dataset~\ref{ass: dataset: 3data not intersect}. The {\color{gray}{gray arrows}} plot the vector field $-\nabla\cL(\cdot)/\norm{\nabla\cL(\cdot)}$; the {\color{red!80!black}{red dashed line}} corresponds to the max-margin solution $\bw^\star$; the {\color{green!80!black}{green zone}} $\bbA$ is an ``attractor'' of NGD dynamics. 
    We plot the trajectories of {\color{violet}{PPGD}} and {\color{orange}{NGD}} for 8 iterations starting from the same initial point $\bw(1)$ (black), where $\bw(1)$ is trained by NGD starting from $\bw(0)=\bzero$ (black).}
    \label{fig: toy dataset}
\end{figure*}

\textit{The Vector Field.} To gain an intuitive understanding of why NGD dynamics is slow in margin maximization for Dataset \ref{ass: dataset: 3data not intersect}, we visualize the direction of normalized gradient in Figure~\ref{fig: 3 data vector field} (the gray arrows).  We can see that the \textbf{centripetal velocity}, i.e., the component orthogonal to $\bw^\star$, becomes tiny in the {\color{green!80!black}{green zone}} $\mathbb{A}$ and diminishes to zero at the green line. Consequently, {\color{orange}NGD} (the orange curve) always enters and is attracted in $\mathbb{A}$ (and remain close the green line). A rigorous analysis in  Appendix~\ref{appendix: proof: methodology} shows that  the ``attractor'' of NGD (i.e., the green zone) is given by $\bbA:=\left\{\bw:{\log2}/{4}\leq w_2\sqrt{1-{\gamma^\star}^2}\leq{3\log2}/{4}\right\}$.

\textit{The Inefficiency of NGD.}  
We can see that while NGD
keeps trapped in the attractor $\mathbb{A}$, $\bw(t)$ moves towards infinity and accordingly, $\hat{\bw}(t)\to \bw^\star$ as $t\to\infty$. However, due to the gap between $\bbA$ and the max-margin direction $\bw^\star$, we have $|w_2(t)|=\Theta(1)$ for all $t\in \mathbb{N}$. Consequently, the directional convergence of NGD is cursed to have a rate at most  
\begin{align*}
    \|\hat{\bw}(t)-\bw^\star\|=& \sqrt{\left(\frac{w_1(t)}{\|\bw(t)\|}-1\right)^2+\left(\frac{w_2(t)}{\|\bw(t)\|}\right)^2}
    \\=&\Theta\left(\frac{|w_2(t)|}{\norm{\bw(t)}}\right)=\Theta\left(\frac{1}{t}\right), 
\end{align*}
where we use the fact that the norm $\|\bw(t)\|$ grows at $\Theta(t)$ rate (Lemma~\ref{thm: NGD upper bound ji}).

\paragraph*{Acceleration via Amplifying the Centripetal Velocity.} In Figure~\ref{fig: 3 data vector field}, we can see that the centripetal velocity is helpful for converging towards $\bw^\star$ when $\bw$ is  away from $\bw^\star$. The inefficiency of NGD stems from the fact that NGD is trapped in the green zone where the centripetal velocity is tiny. Therefore, to accelerate the directional convergence, we can stretch $\bw(t)$ outside the the green zone via {\em rescaling}:  $\bw(t)\to c\bw(t)$ for some  $c>1$. This rescaling does not change the margin (due to the homogeneity) but reposition $\bw(t)$ into a region where the centripetal velocity is lower bounded,  thereby enabling a faster directional convergence when employing  NGD steps there. 

Based the above intuition, we propose the {\em Progressive  Rescaling Gradient Descent} (PRGD) given in Alg.~\ref{alg: PRGD}. The additional projection step in Alg.~\ref{alg: PRGD} is proposed to stablize training  by avoiding 
the rapid explosion of parameters' norm in each cycle. It is shown to be useful in experiments but does not affect our theoretical results (Theorem \ref{thm: PRGD main thm}).

\begin{algorithm}[!ht]
	\caption{Progressive  Rescaling Gradient Descent (PRGD)}
	\label{alg: PRGD}
	\KwIn{Dataset $\cS$; Initialization $\bw(0)$; Progressive Time $\{T_{k}\}_{k=0}^{K}$; Progressive Radius $\{R_{k}\}_{k=0}^{K}$; }  
	\For{$k=0,1,2,\cdots,K$}{
        $\bw(T_{k}+1)={R_k}\frac{\bw(T_{k})}{\norm{\bw(T_{k})}}$; 
         \quad\quad\quad\quad\quad $\triangleright$ {\makecell[c]{progressive \\ rescaling step}}
        \\
        \For{$T_k+1\leq t\leq T_{k+1}-1$}{
            $\bv(t+1)=\bw(t)-\eta\frac{\nabla\cL(\bw(t))}{\cL(\bw(t))}$;\  $\triangleright${\makecell[c]{ normalized grad- \\ ient descent step}}
            \\
		$\bw(t+1)= \Proj_{\bbB(\boldsymbol{0},R_{k})}\bracket{\bv(t+1)}$;
            \hfill $\triangleright$ 
            {\makecell[c]{project-\\ion step}}
            }
        }
    \KwOut{$\bw(T_K+1)$.}
\end{algorithm}

\textit{The Efficiency of PRGD.} 
In Figure~\ref{fig: 3 data vector field}, we plot the trajectory of {\color{violet}{PRGD}} (the purple curve) with hyperparameter $T_{k+1}-T_k=2$. This means that in each cycle, {\color{violet}{PRGD}} executes one step of norm rescaling, followed by one step of projected NGD.  It is evident from the figure that  {\color{violet}{PRGD}}  converges towards $\bw^*$ in direction much faster. This acceleration can be attributed to the following mechanism: the rescaling step allows PRGD to move out of the attractor  $\bbA$ (where the centripetal velocity is tiny) and to undertake NGD in the line $w_2=1$, where 
{\em 
the centripetal velocity has a uniformly positive lower bound.
}
Utilizing this fact, we can show that the norm increases exponentially fast via simple geometric calculation. Consequently, the directional convergence is exponentially fast 
\begin{align*}
    &\|\hat{\bw}(2k-1)-\bw^\star\|=\|\hat{\bw}(2k)-\bw^\star\|
    \\=&\Theta\left(\frac{|w_2(2k)|}{\norm{\bw(2k)}}\right)=\Theta\left(\frac{1}{\norm{\bw(2k)}}\right)=e^{-\Omega(k)}.
\end{align*}

The following proposition formalizes the above intuitive analysis of NGD and PRGD for Dataset \ref{ass: dataset: 3data not intersect}, whose proof is deferred to Appendix~\ref{appendix: proof: methodology}.
\begin{proposition}
\label{thm: 3data}
Consider Dataset~\ref{ass: dataset: 3data not intersect}. Then NGD~\eqref{equ: NGD} can only maximize the margin polynomially fast, while PRGD~(Alg.~\ref{alg: PRGD}) can maximize the margin exponentially fast. Specifically,
\begin{itemize}[leftmargin=2em]
\item \textbf{(NGD).} Let $\bw(t)$ be NGD \eqref{equ: NGD} solution at time $t$  with $\eta=1$ starting from $\bw(0)=\bzero$. Then both the margin maximization and directional convergence are at (tight) polynomial rates:
$\norm{\hat{\bw}(t)-\bw^\star}=\Theta\left({1}/{t}\right),\gamma^{\star}-\gamma(\bw(t))=\Theta\left({1}/{t}\right)$.

\item \textbf{(PRGD).} Let $\bw(1)$ be NGD~\eqref{equ: NGD} solution at time $1$ with $\eta=1$ starting from $\bw(0)=\bzero$, and let $\bw(t)$ be PRGD solution (Alg.~\ref{alg: PRGD}) at time $t$ with $\eta=1$ starting from $\bw(1)$. Then there exists a set of hyperparameters $\{R_k\}_k$ and $\{T_k\}_k$ such that $R_k=e^{\Theta(k)}$ and $T_k=\Theta(k)$, and both the margin maximization and directional convergence are at (tight) exponential rate:
$\norm{\hat{\bw}(t)-\bw^\star}=e^{-\Theta\left(t\right)},\gamma^{\star}-\gamma(\bw(t))=e^{-\Theta\left(t\right)}$.
\end{itemize}
\end{proposition}

\section{Centripetal Velocity Analysis}
\label{section: landscape analysis}

In the above analysis, the key property enabling the acceleration  for Dataset \ref{ass: dataset: 3data not intersect} is the existence a region where the centripetal velocity is uniformly lower bounded and we can stretch $\bw(t)$ to this region by simple norm rescaling. In this section, we demonstrate that this property holds generally. 


We needs the following decomposition of parameter for our fine-grained analysis of directional dynamics. Note that the same decomposition has been employed in~\citet{ji2021characterizing,wu2023implicit}.
\begin{definition}
    Let $\cP(\bw):=\<\bw,\bw^{\star}\>\bw^{\star}$ and $\cP_{\perp}(\bw):=\bw-\<\bw,\bw^{\star}\>\bw^{\star}$. 
    It is worth noting that for any $\bw\in\bbR^d$, we have the following decomposition
    \[
    \bw=\cP(\bw)+\cP_{\perp}(\bw).
    \]
\end{definition}

We can now formally propose the definition of the ``centripetal velocity'' as follows and a visual illustration of the definition is provided in Figure~\ref{fig: centripetal velocity vis}. 
\begin{definition}[Centripetal Velocity]\label{def: centripetal angular velocity}
The normalized gradient at $\bw\in\bbR^d$ is $\nabla\cL(\bw)/\cL(\bw)$ and we define the centripetal velocity $\varphi(\bw)$ at $\bw$ by
\begin{equation*}
    \varphi(\bw):=\<-\frac{\nabla\cL(\bw)}{\cL(\bw)},-\frac{\cP_{\perp}(\bw)}{\norm{\cP_{\perp}(\bw)}}\>.
\end{equation*}
\end{definition}

\begin{figure}
    \centering
    \includegraphics[width=0.33\textwidth]{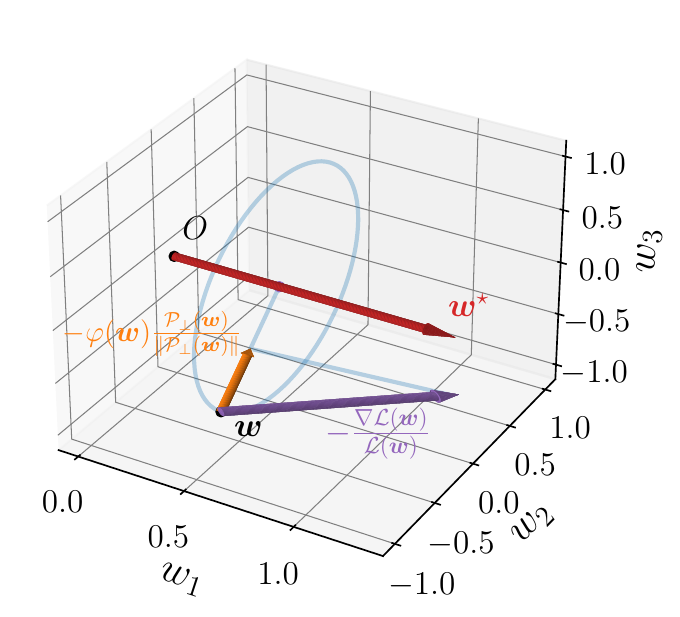}
    \caption{\small A visual illustration of Definition~\ref{def: centripetal angular velocity} in $\bbR^3$. The
    {\color{red!80!black}{red arrow}} corresponds to the max-margin direction $\bw^\star$. At $\bw\in\bbR^3$, the {\color{violet}{purple arrow}} signifies the normalized negative gradient; 
    the {\color{orange}{orange arrow}} depicts the projection of $-\nabla\cL(\bw)/\cL(\bw)$ along the centripetal direction $-\cP_{\perp}(\bw)/\norm{\cP_{\perp}(\bw)}$, reflecting the \textit{centripetal velocity} $\varphi(\bw)$.}
    \label{fig: centripetal velocity vis}
\end{figure}


In addition, our subsequent analysis crucially relies on the following geometry: 
\begin{definition}[Semi-infinite Hollow Cylinder]\label{def: half hollow cylindrical}
    We use 
    \begin{equation}
    \begin{aligned}
        \bbC(D_1,D_2;H)&:=\big\{\bw\in{\rm span}\{\bx_i:i\in[n]\}: 
        \\D_1\leq&\norm{\cP_{\perp}(\bw)}\leq D_2;\left<\bw,\bw^\star\right>\geq H\big\}
    \end{aligned}
    \end{equation}
    to denote a cylinder which starts from the height $H$ and extends infinitely along the direction $\bw^{\star}$.
\end{definition}

Our subsequent analysis will concentrate on this semi-infinite hollow cylinder as PRGD ensures the iterations will be confined in the region. 
Additionally, it is crucial to note that our attention is restricted to the smaller subspace ${\rm span}\{\bx_i:i\in[n]\}$, rather than the entire space $\bbR^d$. This is justified by fact that the trajectories of GD, NGD, and PRGD, when initialized from $\bzero$, will remain staying in this subspace indefinitely.

\subsection{The Existence of a Favorable Semi-infinite Hollow Cylinder}
In this subsection, we undertake a theoretical examination of the centripetal velocity, as defined in Definition~\ref{def: centripetal angular velocity}, on the semi-infinite hollow cylinder described in Definition~\ref{def: half hollow cylindrical}. Our investigation aims to address the following query:
\vspace{.05cm}
\begin{center}
    \textit{Does a ``favorable'' semi-infinite hollow cylinder exist where the centripetal velocity\\ consistently maintains a positive lower bound?}
\end{center}
\vspace{.1cm}

\begin{assumption}[Non-degenerate data
]\label{ass: non-degenerate data}
Let $\cI$ be the index set of the support vectors. Assume there exist $\alpha_i>0$ $(i\in\cI)$ such that $\bw^\star=\sum_{i\in\cI}\alpha_i \bx_i y_i$ and ${\rm span}\{\bx_i:i\in\cI\}={\rm span}\{\bx_i:i\in[n]\}$.
\end{assumption}

Under this assumption, we can establish the existence of  a favorable semi-infinite cylindrical surface as follows. The proof  is deferred to Appendix~\ref{appendix: Centripetal Velocity Analysis}.

\begin{theorem}[Centripetal Velocity Analysis, Main result]
\label{thm: Centripetal Velocity Analysis, main result}
Under Assumption~\ref{ass: linearly separable} and~\ref{ass: non-degenerate data}, there exists a semi-infinite hollow cylinder $\bbC(D,2D;H)$ and a positive constant $\mu>0$ such that
\begin{align*}
    \inf_{\bw\in\bbC(D,2D;H)}\varphi(\bw)\geq\mu.
\end{align*}
\end{theorem}

Assumption \ref{ass: non-degenerate data} has been widely used in prior analysis of the margin-maximization bias of gradient-based algorithms. 
The first part, requiring strictly positive dual variables to ensure directional convergence, is rather weak and holds for almost all linearly separable data \citep{soudry2018implicit,nacson2019convergence,ji2021characterizing,wang2021momentum}.
The second part requires the support vectors to span the entire dataset. This condition has been adopted  in refined analysis of the residual in Theorem 4 of~\citep{soudry2018implicit}, the analysis of SGD dynamics~\citep{nacson2019stochastic}, and the edge of stability~\citep{wu2023implicit}. We emphasize that the second part is only a technical condition as experimental results for real-world datasets in Section~\ref{section: experiments I} demonstrate that PRGD performs effectively even in cases where this condition is not met. We leave the relaxation of this condition for future work.

\section{Convergence Analysis}
\label{section: convergence analysis}

\subsection{Exponentially Fast Margin Maximization of PRGD}

 Theorem~\ref{thm: Centripetal Velocity Analysis, main result} ensures the existence of a favorable semi-infinite hollow cylinder. The following theorem shows that PRGD (Alg.~\ref{alg: PRGD}) can leverage the favorability to achieve an exponential  rate in directional convergence and  margin maximization.


\begin{theorem}[PRGD, Main Result]\label{thm: PRGD main thm} 
Suppose that Assumption~\ref{ass: linearly separable} and~\ref{ass: non-degenerate data} hold. Let $\bw(t)$ be solution generated by the following \textbf{two-phase} algorithms starting from $\bw(0)=\bzero$:
\begin{itemize}[leftmargin=2em]
    \item \textit{Warm-up Phase}: Run GD~\eqref{equ: GD} or NGD~\eqref{equ: NGD} with $\eta\leq 1$ for $T_{\mathrm{w}}$ steps starting from $\bw(0)$;
    \item \textit{Acceleration Phase}: Run PRGD (Alg.~\ref{alg: PRGD}) with some $\eta,\{R_k\}_k,\{T_k\}_k$ starting from $\bw(T_{\mathrm{w}})$.
\end{itemize}
Then there exist a set of hyperparameters $\eta=\Theta(1)$, $R_k=e^{\Theta(k)}$ and $T_k=\Theta(k)$, ensuring both directional convergence and margin maximization occur at \textbf{exponential} rates:
\begin{align*}
    \norm{\hat{\bw}(t)-\bw^\star}= e^{-\Omega(t)};\quad \gamma^\star-\gamma(\bw(t))= e^{-\Omega(t)}.
\end{align*}
\end{theorem}

This theorem demonstrates that PRGD can achieve both directional convergence and margin maximization exponentially fast. In stark contrast, all existing algorithms maximize the margin at notably slower rates, including $\cO(1/\log t)$ for GD~\citep{soudry2018implicit}, $\cO(1/t)$ for NGD~\citep{ji2021characterizing}, and $\tilde{\cO}(1/t^2)$ for Dual Acceleration~\citep{ji2021fast,wang2022accelerated}.

The complete proof of Theorem \ref{thm: PRGD main thm} is deferred to Appendix~\ref{appendix: proof: PPGD} and here, we provide a  proof sketch to illustrate the intuition behind:
\begin{itemize}[leftmargin=2em]
\item The initial warm-up phase utilizes GD to secure a preliminary  directional convergence, albeit at a slower rate, such that the condition $\norm{\hat{\bw}(T_{\mathrm{w}})-\bw^\star}<\min
\{D/2H,1/2\}$ is satisfied. This condition is crucial as it allows for the subsequent stretching of $\bw(T_{\mathrm{w}})$ to the favorable semi-infinite hollow cylinder $\bbC(D,2D;H)$ (in Theorem~\ref{thm: Centripetal Velocity Analysis, main result}) through a straightforward norm rescaling. Without this condition, rescaling cannot reposition $\bw$ into the favorable region.


\item Following the warm-up phase, rescaling is employed to position $\bw(T_{\mathrm{w}})$ into the favorable semi-infinite hollow cylinder $\bbC(D,2D;H)$ by choosing $R_1=\frac{D}{\norm{\cP_{\perp}(\bw(T_{\mathrm{w}}))}}$. 
NGD steps taken thereafter yield a significant directional convergence, however, NGD steps also drive solutions to leave away from $\bbC(D,2D;H)$. To overcome this issue, 
by setting a sequence of progressively increasing radii $\{R_k\}$, we can reposition the parameter back to $\bbC(D,2D;H)$ again, as is evident illustrated in Figure~\ref{fig: 3 data vector field}. 
Lastly, through a simple geometric calculation, we can demonstrate that such directional convergence is exponentially fast. 
\end{itemize}

\begin{remark}
We  clarify that Theorem ~\ref{thm: PRGD main thm} establishes the exponentially fast margin maximization of PRGD under a particular family of hyperparameters. For a broader range of hyperparameter choices, we delve into experimetnal explorations in Section~\ref{section: experiments}.
\end{remark}

\begin{remark}
In Proposition~\ref{thm: 3data}, we have provided a tightly exponentially fast rate on Dataset~\ref{ass: dataset: 3data not intersect}, which satisfies Assumption~\ref{ass: linearly separable} and ~\ref{ass: non-degenerate data}. Hence, the tightness of Theorem~\ref{thm: PRGD main thm} is ensured.
\end{remark}



\subsection{Inefficiency of GD and NGD}
In order to theoretically justify RPGD's superiority over GD and NGD,  we need lower bounds of directional convergence and margin maximization rates for GD and NGD. However,  as mentioned above, previous studies have only established the upper bounds of GD and NGD as $\mathcal{O}(1/\log t)$ and $\mathcal{O}(1/t)$, respectively. In this section, we further identify mild assumptions, under which we show that those rates also serve as the lower bounds for GD and NGD.

\begin{theorem}[GD and NGD, Main Results]\label{thm: GD NGD main result}
Suppose Assumption~\ref{ass: linearly separable} and~\ref{ass: non-degenerate data} hold. Additionally, we assume $\gamma^\star\bw^\star\ne\frac{1}{|\cI|}\sum_{i\in\cI}\bx_i y_i$. Then,
 \begin{itemize}[leftmargin=2em]
    \item For GD \eqref{equ: GD} with $\eta\leq \eta_0$ starting from $\bw(0)=\bzero$ (where $\eta_0$ is a constant), it holds that
    \begin{align*}
        \norm{\hat{\bw}(t)-\bw^\star}=&\Theta(1/\log t);
        \\\gamma^\star-\gamma(\bw(t))=&\Omega\bracket{1/\log^2 t}.
    \end{align*}
    \item   
    For NGD~\eqref{equ: NGD} with $\eta\leq \eta_0$ starting from $\bw(0)=\bzero$ (where $\eta_0$ is a constant), there exists a subsequence $\bw(t_k)$ ($t_k\to\infty$) such that
    \begin{align*}
        \norm{\hat{\bw}(t_k)-\bw^\star}=&\Theta(1/t_k);
        \\\gamma^\star-\gamma(\bw(t_k))=&\Omega\bracket{1/t_k^2}.
    \end{align*}
\end{itemize}

\end{theorem}

To the best of our knowledge, Theorem~\ref{thm: GD NGD main result} provides the {\bf first} lower bounds of both directional convergence and margin maximization for GD and NGD. We anticipate that this result can help build intuitions for future analysis of the implicit bias and convergence of gradient-based algorithms.

As presented in Table~\ref{table: comparison tight}, under the same conditions--Assumption~\ref{ass: linearly separable},~\ref{ass: non-degenerate data}, and $\gamma^\star\bw^\star\ne\frac{1}{|\cI|}\sum_{i\in\cI}\bx_i y_i$, PRGD can achieve directional convergence {\em exponentially fast} with the rate $e^{-\Omega(t)}$. In contrast, Theorem~\ref{thm: GD NGD main result} ensures that GD exhibits a {\em tight} bound with exponentially slow rate $\Theta(1/\log t)$, as well as NGD maintains a  tight bound of polynomial speed $\Theta(1/t_k)$. Moreover, for margin maximization, Theorem~\ref{thm: GD NGD main result} also provides {\em nearly tight} lower bounds for GD and NGD. 


The detailed proof of Theorem~\ref{thm: GD NGD main result} is provided in Appendix~\ref{appendix: proof of hardness of GD and NGD}. 
While the proof of the lower bounds involves a more intricate convex optimization analysis compared to Proposition~\ref{thm: 3data} (especially for NGD due to its aggressive step size), the fundamental insights shared by both proofs are remarkably similar. 
Specifically, for NGD, our conditions ensures the existence of a nearly ``attractor'' region, such that (i) there exists a sequence of NGD $\bw(t_k)$ $(t_k\to\infty)$ falls within this region; (ii) the condition $\gamma^\star\bw^\star\ne\frac{1}{|\cI|}\sum_{i\in\cI}\bx_i y_i$ ensures a $\Omega(1)$ distance between the attractor region and the max-margin direction. 
Since the norm grows as $\norm{\bw(t_k)}=\Theta(t_k)$, NGD is cursed to have only $\Omega(1/\norm{\bw(t_k)})=\Omega(1/t_k)$ \text{directional convergence rate}.

\begin{figure*}[!hb]
    \vspace{.2cm}
    \centering
    {
    \includegraphics[width=0.24\textwidth]{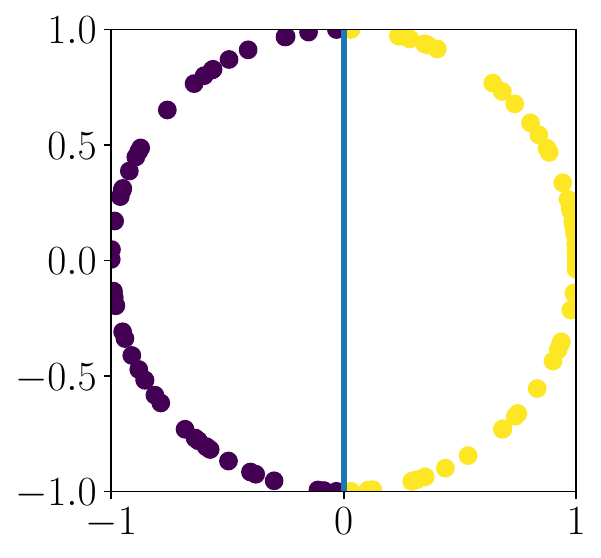}
    \hspace*{.2em}
    \includegraphics[width=0.22\textwidth]{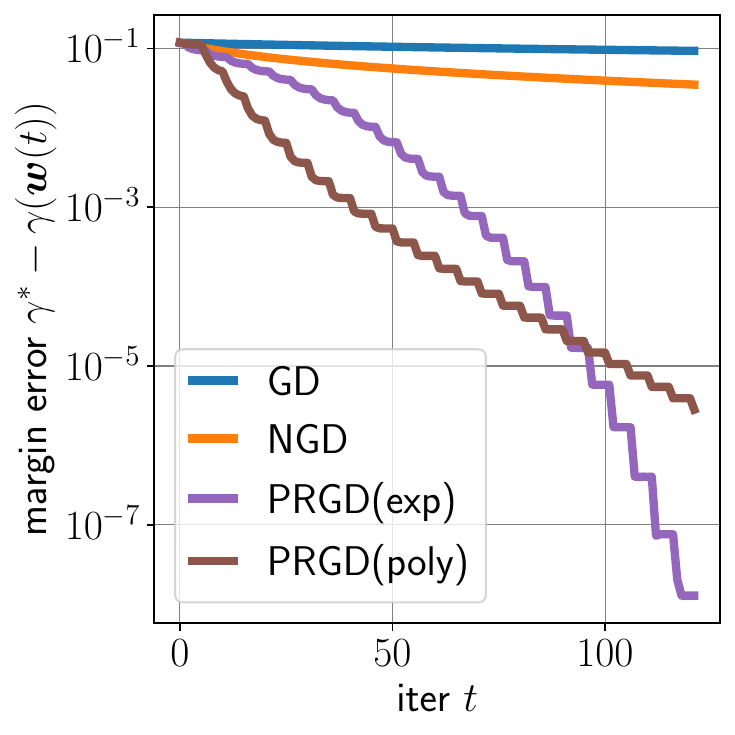}
    \hspace*{.5em}
    \includegraphics[width=0.22\textwidth]{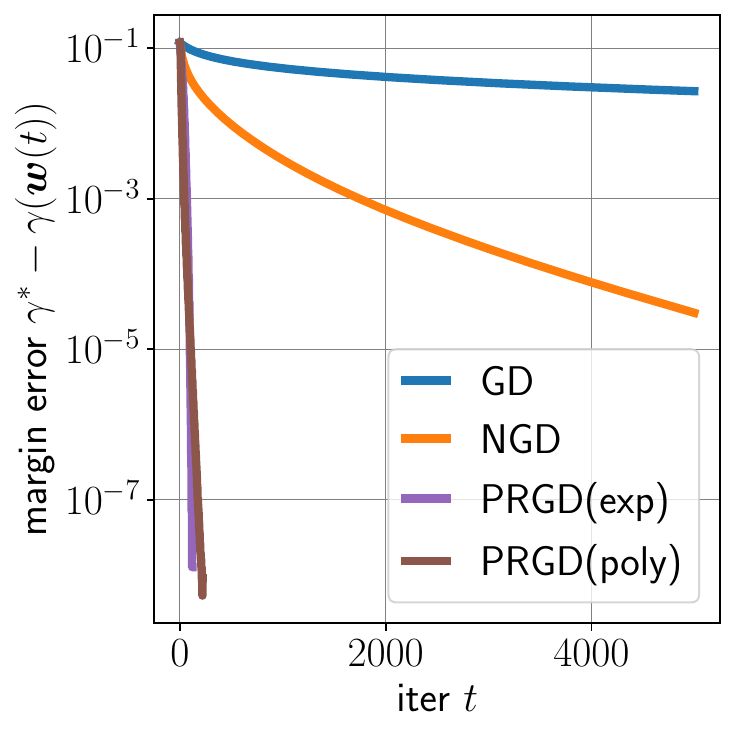}
    }
    \caption{\small Comparison of margin maximization rates of different algorithms on a synthetic dataset. (left) A visualization of the 2d synthetic dataset. The yellow points represent the data with label $1$, while the purple points corresponds to the data with label $1$; (middle)(right) The comparison of margin maximization rates of different algorithms on this dataset at small and large time scales, respectively.}
    \label{fig: synthetic data}
\end{figure*}

\section{Numerical Experiments}
\label{section: experiments}

\subsection{Linearly Separable Datasets}
\label{section: experiments I}

{\bf Experiments on Synthetic Dataset.} 
We start our experimental validations with two synthetic linearly separable datasets. For synthetic datasets, the value of  $\gamma^\star$ is explicit, and as such, we can explicitly compute the margin gap.
To ensure a fair comparison, we maintain the same step size $\eta=1$ for all GD, NGD, and PRGD. 

While our theoretical analysis (Theorem~\ref{thm: PRGD main thm}) is confined to a specific set of hyper-parameters, it is worth noting that Theorem~\ref{thm: Centripetal Velocity Analysis, main result}, a crucial property used in the proof of Theorem~\ref{thm: PRGD main thm}, holds over a relatively broad region. This flexibility enables a simpler selection of hyperparameters.
Following the guidelines provided in Theorem~\ref{thm: PRGD main thm}, we employ PRGD(exp) with hyperparameters: 
\begin{align*}
    T_{k+1}-T_k\equiv5, \ R_k=R_0\times 1.2^k.
\end{align*} 
To illustrate the impact of the progressive radius, we also examine PRGD(poly) configured with
\begin{align*}
    T_{k+1}-T_k\equiv5, \ 
    R_k=R_0\times k^{1.2},
\end{align*}
where the progressive radius increases polynomially. For more experimental details, refer to Appendix~\ref{appendix: experimental details synthetic}. 
Some of the experimental results are provided in Figure~\ref{fig: synthetic data}, and the complete results are referred to Appendix~\ref{appendix: experimental details synthetic}. Consistent with Theorem~\ref{thm: PRGD main thm}, PRGD(exp) indeed maximizes the margin (super-)exponentially fast, and surprisingly, PRGD(poly) also performs relatively well for this task. In contrast, NGD and GD reduce the margin gaps much more slowly, which substantiates Theorem~\ref{thm: GD NGD main result}.

\textbf{Experiments on Real-World Datasets.}
In this case, we extend our experiments to real-world datasets.
Specifically, we employ the \texttt{digit} datasets from \texttt{Sklearn}, which are image classification tasks with $d=64$, $n=300$.
In this real-world setting, we lack prior knowledge of the exact $\gamma^\star$.
Instead, we approximate $\gamma^\star$ by employing $\gamma(\bw(t))$ obtained by a sufficiently trained NGD. 
Additionally, it is worth noting that these datasets do not satisfy the second part of Assumption~\ref{ass: non-degenerate data}. For instance, in the \texttt{digit-01} dataset, $\rank\{\bx_i:i\in\cI\}=2<\rank\{\bx_i:i\in[n]\}=51$, where $\rank\{\bx_i:i\in\cI\}$ is calculated by approximating $\bw^\star$ using $\hat{\bw}(t)$, obtained through training NGD sufficiently.

In real experiments, we test both PRGD(exp) and PRGD(poly) and consistently observe that the latter performs much better. Therefore, in this experiment, we employ a modified variant of PRGD with smaller progressive radii: 
\begin{equation}\label{eqn: qq1}
R_k=R_0 \cdot k^\alpha,\quad T_{k+1}-T_k=T_0 \cdot k^\beta,
\end{equation}
where $\alpha, \beta$ are hyperparameters to be tuned.

The results with well-tuned hyperparameters $\alpha=\beta=0.6$ are presented in Figure~\ref{fig: digit}. It is evident that, in these real-world datasets, PRGD consistently beats GD and NGD in terms of margin maximization rates.

\begin{figure}[!ht]
    \centering
    \includegraphics[width=3.7cm]{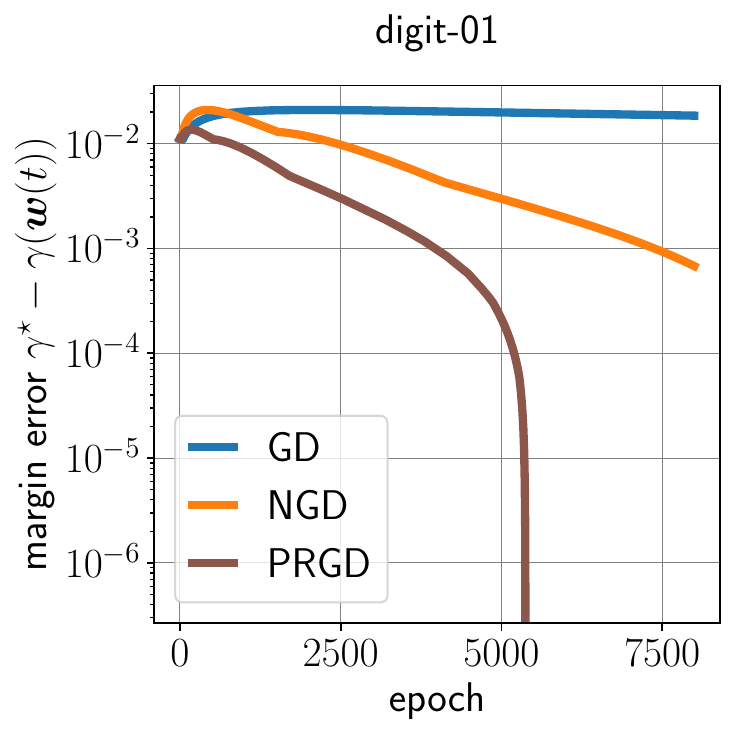}
    \hspace{0.2cm}
    \includegraphics[width=3.7cm]{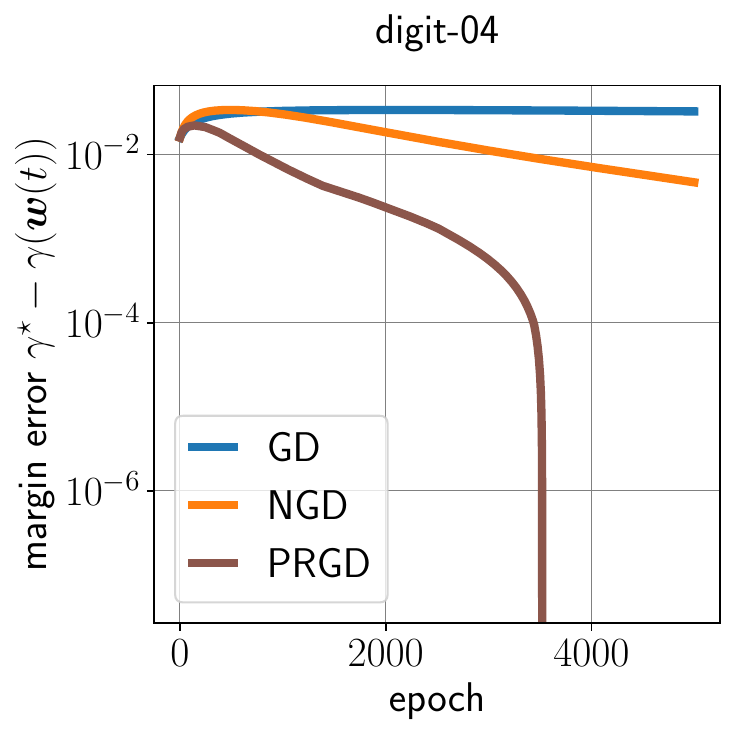}
    \caption{\small Comparison of margin maximization rates of different algorithms on \texttt{digit} (real-word) datasets. (Left) the results on \texttt{digit-01} dataset; (Right) the results on \texttt{digit-04} dataset.}
    \label{fig: digit}
\end{figure}

\subsection{Linearly Non-separable Datasets and Deep Neural Networks}
In this subsection, we further explore the potential practical utilities of PRGD for datasets that are not linearly separable. 
1) In the {\bf first} experiment, we still consider linear models but for classifying a linearly non-separable dataset, \texttt{Cancer} in \texttt{Sklearn}.
2) For the {\bf second} experiment, we examine the performance of PRGD for  deep neural networks. 
Inspired by~\citet{lyu2019gradient,ji2020directional}, the max-margin bias also exists for homogenized neural networks. Thus, we follow~\citet{lyu2019gradient} and examine our algorithm for homogenized VGG-16 network~\citep{vgg} on the full CIFAR-10 dataset~\citep{krizhevsky2009learning}, without employing any explicit regularization. 
Additionally, in this setting, we employ  mini-batch stochastic gradient  instead of the full gradient for these algorithms, and we also fine-tune the learning rates of GD, NGD, and PRGD. Both NGD and PRGD share the same learning rate scheduling strategy as described in~\citet{lyu2019gradient}. 
For both experiments,  we follow the same strategy as described in~\eqref{eqn: qq1} to tune the progressive hyperparameters of PRGD. For more experimental details, please refer to Appendix~\ref{appendix: experimental details VGG}.

The experimental results are presented in Fig~\ref{fig: cancer} and Fig~\ref{fig: vgg}, respectively. One can see that our PRGD algorithm archives better generalization performance and outperforms GD and NGD for both tasks.

\begin{figure}[ht]
    \centering
    \subfloat[Linear Model on \texttt{Cancel}]{\label{fig: cancer}
    \includegraphics[width=4.18cm]{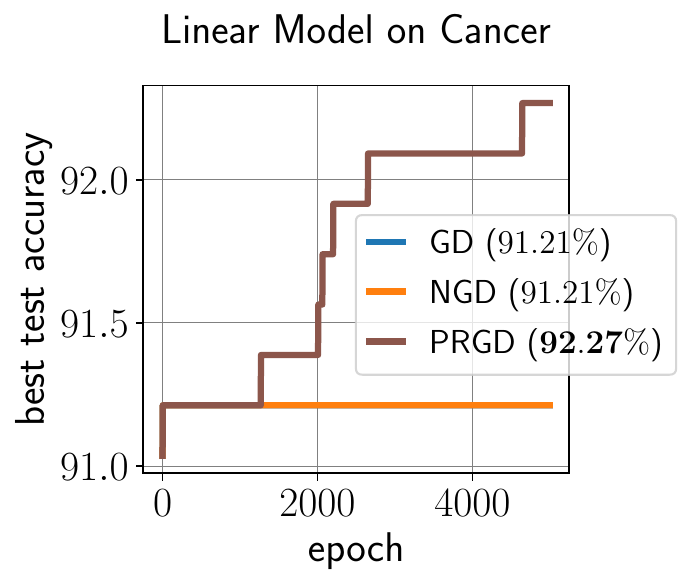}}
    \subfloat[VGG on Cifar-10]{\label{fig: vgg}
    \includegraphics[width=4.1cm]{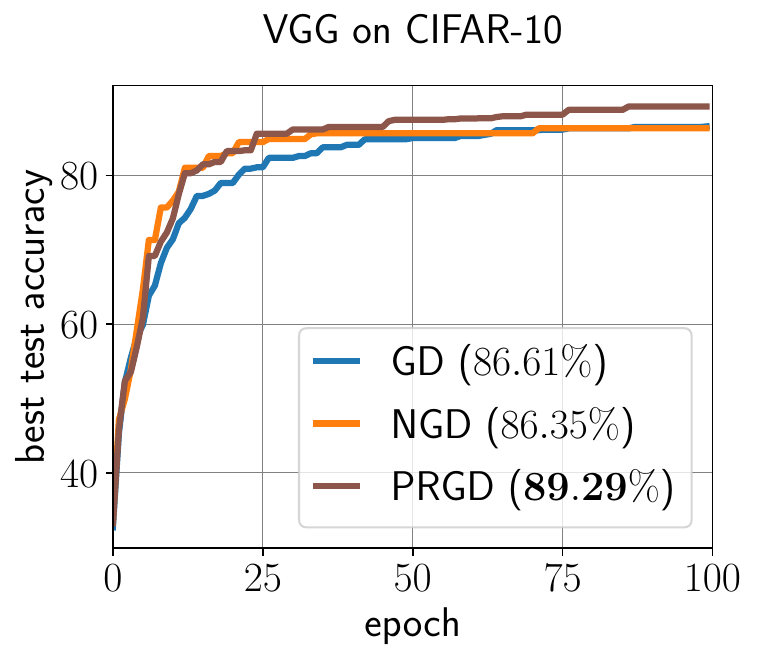}}
    \caption{\small Comparison of the generalization performance of GD, NGD, and PRGD for non-linearly separable datasets and deep neural networks.}
\end{figure}

\section{Concluding Remark}
In this work, we investigate the mechanisms driving the convergence of gradient-based algorithms towards max-margin solutions. Specifically, we elucidate why GD and NGD can only achieve polynomially fast margin maximization by examining the properties of the velocity field linked to (normalized) gradients. This analysis inspires the design of a novel algorithm called PRGD that significantly accelerates the process of margin maximization. To substantiate our theoretical claims, we offer both synthetic and real-world experimental results, thereby underscoring the potential practical utility of the proposed PRGD algorithm. Looking ahead, an intriguing avenue for future research is the application of progressive norm rescaling techniques to state-of-the-art real-world models. In addition, it would be worthwhile to explore how PRGD can cooperate with other explicit regularization techniques, such as batch normalization, dropout, and sharpness-aware minimization~\citep{foret2020sharpness}, to further improve the generalization performance.

\section*{Acknowledgements}

Mingze Wang is supported in part by the National Key Basic Research Program of China (No. 2015CB856000). Lei Wu is supported in part by a startup fund from Peking University. We thank Dr. Chao Ma, Zihao Wang, Zhonglin Xie for helpful discussions and anonymous reviewers for their valuable suggestions.

\section*{Impact Statement}
This paper presents work whose goal is to advance the theoretical aspects of Machine Learning and, specifically, the margin maximization bias of gradient-based algorithms.
There are many potential societal consequences of our work, none which we feel must be specifically highlighted here.



\newpage
\appendix
\onecolumn
\newpage
\appendix

\begin{center}
    \noindent\rule{\textwidth}{1.0pt} 
    \vspace{-0.25cm}
    \LARGE \textbf{Appendix} 
    \noindent\rule{\textwidth}{1.0pt}
\end{center}

\startcontents[sections]
\printcontents[sections]{l}{1}{\setcounter{tocdepth}{2}}


\vspace{1.cm}
\section{Proofs in Section~\ref{section: methodology}}
\label{appendix: proof: methodology}

\subsection{Proof of Proposition~\ref{thm: 3data}}

\begin{proof}[Proof of Proposition~\ref{thm: 3data}]\ \\ 
\underline{Step I. Regularized path analysis.}

For simplicity, we denote $\bz_1=\bx_1y_1$ and $\bz_2=\bx_2y_2$.
\begin{align*}
    \cL(\bw)=\frac{1}{3}\bracket{2e^{-\bw^\top\bz_1}+e^{-\bw^\top\bz_2}}=\frac{1}{3}e^{-w_1\gamma}\bracket{2e^{-w_2\sqrt{1-\gamma^2}}+e^{w_2\sqrt{1-\gamma^2}}}.
\end{align*}
\begin{align*}
    \nabla \cL(\bw)=\begin{pmatrix}
    -\frac{1}{3}e^{-w_1\gamma}\gamma\bracket{2e^{-w_2\sqrt{1-\gamma^2}}+e^{w_2\sqrt{1-\gamma^2}}}\\
    \frac{1}{3}e^{-w_1\gamma}\sqrt{1-\gamma^2}\bracket{-2e^{-w_2\sqrt{1-\gamma^2}}+e^{w_2\sqrt{1-\gamma^2}}}
    \end{pmatrix}.
\end{align*}

For any fixed $R>0$, we will calculate the regularized solution in the ball $\norm{\bw}_2\leq R$.

From the expression of $\nabla\cL(\bw)$, we know $\nabla\cL(\bw)\ne\bzero$ for any $\bw\in\bbR^d$. Hence, it must holds $\norm{\bw_{\rm reg}^*(R)}_2=R$. Moreover, we can determine the signal of $w_{\rm reg,1}^*(R)$ and $w_{\rm reg,2}^*(R)$. From the symmetry of the $\ell_2$ ball, we know $w_{\rm reg,1}^*(R)<0$ and $w_{\rm reg,2}^*(R)>0$. This is because: if $w_{\rm reg,1}^*(R)>0$, then $\cL(-w_{\rm reg,1}^*(R),w_{\rm reg,2}^*(R))<\cL(w_{\rm reg,1}^*(R),w_{\rm reg,2}^*(R))$, which is contradict to the optimum of $\bw_{\rm reg}^*(R)$. 

Then from the optimum and differentiability, we have
\begin{align*}
    \frac{\<\bw_{\rm reg}^*(R),-\nabla\cL\bracket{\bw_{\rm reg}^*(R)}\>}{R\norm{\nabla\cL\bracket{\bw_{\rm reg}^*(R)}}_2}=1,
\end{align*}
which means
\begin{align*}
\bw_{\rm reg}^*(R)\pll\nabla\cL\bracket{\bw_{\rm reg}^*(R)},\quad\<\bw_{\rm reg}^*(R),\nabla\cL\bracket{\bw_{\rm reg}^*(R)}\><0.
\end{align*}
For simplicity, we use the notation $w_1(R):=w_{\rm reg,1}^*(R)$, $w_2(R):=w_{\rm reg,2}^*(R)$ in the proof below.

By a straightforward calculation and taking the square, we have
\begin{align*}
    \frac{(1-\gamma^2)\bracket{e^{2w_2(B)\sqrt{1-\gamma^2}}+4e^{-2w_2(B)\sqrt{1-\gamma^2}}-4}}{\gamma^2\bracket{e^{2w_2(R)\sqrt{1-\gamma^2}}+4e^{-2w_2\sqrt{1-\gamma^2}}+4}}=\frac{w_2^2(B)}{w_1^2(R)}=\frac{w_2^2(R)}{R^2-w_2^2(R)},
\end{align*}
which is equivalent to
\begin{equation}\label{equ: proof of Theorem 3data: explicit expression}
    \begin{aligned}
    \frac{R^2}{w_2^2(R)}=\frac{1}{1-\gamma^2}+\frac{8\gamma^2}{(1-\gamma^2)\bracket{e^{2w_2(R)\sqrt{1-\gamma^2}}+4e^{-2w_2(R)\sqrt{1-\gamma^2}}-4}}.
\end{aligned}
\end{equation}

With the help of Lemma \ref{lemma: ji regularization path convergence}, we know 
\begin{align*}
    \lim_{R\to\infty}\<\bw^*,\frac{\bw_{\rm reg}^*(R)}{R}\>=\lim_{R\to\infty}\frac{w_1(R)}{\sqrt{w_1^2(R)+w_2^2(R)}}=1,
\end{align*}
which means $\lim\limits_{R\to\infty}\frac{w_2^2(R)}{R^2}=0$. Then taking $R\to\infty$ in \eqref{equ: proof of Theorem 3data: explicit expression}, we have
\begin{align*}
    \lim_{R\to\infty}\bracket{e^{2w_2(R)\sqrt{1-\gamma^2}}+4e^{-2w_2(R)\sqrt{1-\gamma^2}}}=4.
\end{align*}
A straight-forward calculation gives us
\begin{align*}
    \lim_{R\to\infty}w_2(R)=\frac{\log2}{2\sqrt{1-\gamma^2}}.
\end{align*}


\underline{Step II. Proof for NGD.} 

Following the proof, we have
\begin{align*}
    -\frac{\nabla\cL(\bw)}{\cL(\bw)}
    =\begin{pmatrix}
        \gamma \\
        \sqrt{1-\gamma^2}\bracket{2-e^{2w_2\sqrt{1-\gamma^2}}}\big/\bracket{2+e^{2w_2\sqrt{1-\gamma^2}}}
    \end{pmatrix}.
\end{align*}

For NGD, it holds that:
\begin{align*}
    w_1(t+1)&=w_1(t)+\gamma,
    \\
    w_2(t+1)&=w_2(t)+\sqrt{1-\gamma^2}\bracket{2-e^{2w_2(t)\sqrt{1-\gamma^2}}}\big/\bracket{2+e^{2w_2(t)\sqrt{1-\gamma^2}}}.
\end{align*}
It is worth noticing that the dynamics of $w_1(t)$ and $w_2(t)$ are decoupled. For $w_1(t)$, it is easy to verify that $w_1(t)=\gamma t$, $\forall t\geq1$. As for $w_2(t)$, we will estimate the uniform upper and lower bounds.

For simplicity, we denote $x(t):=2w_2(t)\sqrt{1-\gamma^2}-\log2$. From the dynamics of $w_2(t)$, the dynamics of $x(t)$ are
\begin{align*}
    x(t+1)=x(t)+2(1-\gamma^2)\frac{1-e^{x(t)}}{1+e^{x(t)}}=x(t)+2(1-\gamma^2)\bracket{\frac{2}{1+e^{x(t)}}-1}.
\end{align*}
Then we will prove that $|x(t)|\leq\frac{1}{2}\log2$ holds for $t\geq1$ by induction. 

From $x(0)=-\log2$, we have $x(1)=-\log2+\frac{2(1-\gamma^2)}{3}\in\left[-\frac{1}{2}\log2,\frac{1}{2}\log2\right]$.

Assume that $x(t)\in\left[-\frac{1}{2}\log2,\frac{1}{2}\log2\right]$ holds for any $t\leq k$, and we denote $h(x):=x+2(1-\gamma^2)\left(\frac{2}{1+e^x}-1\right)$. Then with the help of Lemma~\ref{lemma: increase 3 data NGD}, the following estimate holds for $t=k+1$:
\begin{align*}
    x(k+1)=&h\left(x(k)\right)\leq h\left(\frac{1}{2}\log2\right)=\frac{1}{2}\log 2+2(1-\gamma^2)\frac{1-\sqrt{2}}{1+\sqrt{2}}<\frac{1}{2}\log 2;
    \\
     x(k+1)=&h\left(x(k)\right)\geq h\left(-\frac{1}{2}\log2\right)=-\frac{1}{2}\log 2+2(1-\gamma^2)\frac{\sqrt{2}-1}{\sqrt{2}+1}>-\frac{1}{2}\log 2.
\end{align*}

By induction, we have proved that $x(t)\in\left[-\frac{1}{2}\log2,\frac{1}{2}\log2\right]$ holds for any $t\geq1$. This implies that $w_2(t)\in\left[\frac{\log2}{4\sqrt{1-\gamma^2}},\frac{3\log2}{4\sqrt{1-\gamma^2}}\right]$ holds for any $t\geq1$.
Hence, 
\begin{align*}
    \frac{\log2}{4t\gamma\sqrt{1-\gamma^2}}\leq\frac{w_2(t)}{w_1(t)}\leq\frac{3\log2}{4t\gamma\sqrt{1-\gamma^2}},\ \forall t\geq 1.
\end{align*}
From the definition of directional convergence, we have:
\begin{align*}
    &\norm{\frac{\bw(t)}{\norm{\bw(t)}}-\bw^\star}=\sqrt{2\left(1-\<\frac{\bw(t)}{\norm{\bw(t)}},\be_1\>\right)}
    =\sqrt{2\left(1-\frac{w_1(t)}{\sqrt{w_1^2(t)+w_2^2(t)}}\right)}
    \\=&
    \sqrt{2\left(1-\frac{1}{\sqrt{\frac{w_2^2(t)}{w_1^2(t)}+1}}\right)}=\Theta\left(\left|\frac{w_2(t)}{w_1(t)}\right|\right)=\Theta\left(\frac{1}{t}\right).
\end{align*}

From the definition of margin, we have:
\begin{align*}
    &\gamma(\bw(t))-\gamma^\star=\min_{i\in[2]}\left<\frac{\bw(t)}{\norm{\bw(t)}},\bz_i\right>-\gamma^\star=\left<\frac{\bw(t)}{\norm{\bw(t)}},\bz_2\right>-\gamma^\star
    \\=&\frac{w_1(t)\gamma^\star-w_2(t)\sqrt{1-{\gamma^\star}^2}}{\sqrt{w_1^2(t)+w_2^2(t)}}-\gamma^\star
    =\frac{-\frac{w_2(t)}{w_1(t)}\sqrt{1-{\gamma^\star}^2}+\gamma^\star}{\sqrt{\frac{w_2^2(t)}{w_1^2(t)}+1}}-\gamma^\star
    \\=&-\frac{\frac{w_2(t)}{w_1(t)}\sqrt{1-{\gamma^\star}^2}}{\sqrt{\frac{w_2^2(t)}{w_1^2(t)}+1}}+\gamma^\star\left(\left(\frac{w_2^2(t)}{w_1^2(t)}+1\right)^{-1/2}-1\right)
    =-\Theta\left(\frac{w_2(t)}{w_1(t)}\right)-\Theta\left(\frac{w_2^2(t)}{w_1^2(t)}\right)
    \\=&-\Theta\left(\frac{w_2(t)}{w_1(t)}\right)=-\Theta\left(\frac{1}{t}\right).
\end{align*}

\underline{Step II. Proof for PRGD.}

For PRGD, to maximize margin exponentially fast, we only need to select $R_k=e^{\Theta(k)}$ and $T_k=\Theta(k)$.
Notice that the choices of $R_k$ and $T_k$ are not unique. For simplicity, we use the following choice to make our proof clear.
\begin{itemize}
    \item Phase I. We run NGD with $\eta=1$ for one step.
    \item Phase II. We run PRGD with $\eta=1$ for $t\geq 1$. Specifically, we select $T_k$ and $R_k$ such that: 
    \begin{gather*}
        T_0=1;\quad T_{k+1}=T_k+2,\ \forall k\geq 0;\quad
        R_{k}=\frac{\norm{\bw(T_k)}}{w_2(T_k)},\ \forall k\geq 0.
    \end{gather*}
\end{itemize}

Recalling our proof for NGD, at the end of Phase I, it holds that $w_2(1)\in\left[\frac{\log2}{4\sqrt{1-{\gamma^\star}^2}},\frac{3\log2}{4\sqrt{1-{\gamma^\star}^2}}\right]$ and $w_1(1)=\gamma^\star$. Then we analyze Phase II. For simplicity, we denote an abosulte constant $q={1+\sqrt{1-{\gamma^\star}^2}\bracket{2-e^{2\sqrt{1-{\gamma^\star}^2}}}\big/\bracket{2+e^{2\sqrt{1-{\gamma^\star}^2}}}}\in(0,1)$.

\begin{itemize}
    \item (S1). $w_2(2k+2)=1$ and $\norm{\bw(2k+2)}=R_k$ hold for any $k\geq0$;
    \item (S2). $\frac{w_2(2k+2)}{w_1(2k+2)}=\frac{w_2(2k+1)}{w_2(2k+1)}$ holds for any $k\geq0$;
    \item (S3). $w_1(2k+2)=\bracket{\frac{1}{q}}^{k}\bracket{w_1(2)+\frac{\gamma^\star}{1-q}}-\frac{\gamma^\star}{1-q}=e^{\Theta(k)}$.
    \item (S4). $R_k=e^{\Theta(k)}$.
    \item (S5). $\frac{w_2(t)}{w_1(t)}=e^{-\Theta(t)}$.
\end{itemize}

According to the update rule of Algorithm~\ref{alg: PRGD}, (S1)(S2) hold directly. 

Then we will prove (S3). Recalling the update rule, for $2k+3$, it holds that
\begin{align*}
    &w_2(2k+3)=w_2(2k+2)+\sqrt{1-{\gamma^\star}^2}\bracket{2-e^{2w_2(2k+2)\sqrt{1-{\gamma^\star}^2}}}\big/\bracket{2+e^{2w_2(2k+2)\sqrt{1-{\gamma^\star}^2}}}
    \\=&1+\sqrt{1-{\gamma^\star}^2}\bracket{2-e^{2\sqrt{1-{\gamma^\star}^2}}}\big/\bracket{2+e^{2\sqrt{1-{\gamma^\star}^2}}}:=q;
\end{align*}
\begin{align*}
    w_1(2k+3)=w_1(2k+2)+\gamma^\star.
\end{align*}
Recalling (S2), we have:
\begin{align*}
    w_1(2k+4)=w_2(2k+4)\frac{w_1(2k+3)}{w_2(2k+3)}=1\cdot\frac{w_1(2k+2)+\gamma^\star}{q}=\frac{1}{q}\left(w_1(2k+2)+\gamma^\star\right),
\end{align*}
where $1+\sqrt{1-{\gamma^\star}^2}\bracket{2-e^{2\sqrt{1-{\gamma^\star}^2}}}\big/\bracket{2+e^{2\sqrt{1-{\gamma^\star}^2}}}\in(0,1)$.

Consequently, a simple calculation can imply (S3):
\begin{align*}
     w_1(2k+2)=\bracket{\frac{1}{q}}^{k}\bracket{w_1(2)+\frac{\gamma^\star}{1-q}}-\frac{\gamma^\star}{1-q}=e^{\Theta(k)},\quad \forall k\geq0.
\end{align*}

Then using (S1) and (S3), (S4) holds:
\begin{align*}
    R_k=\norm{\bw(2k+2)}=\sqrt{1+w_1^2(2k+2)}=e^{\Theta(k)}.
\end{align*}

By (S1)(S2) and (S3), we have:
\begin{align*}
    \frac{w_2(2k+1)}{w_1(2k+1)}=\frac{w_2(2k+2)}{w_1(2k+2)}=\frac{1}{w_1(2k+2)}=e^{-\Theta(2k+2)},
\end{align*}
which implies (S4).

From the definition of directional convergence, we have:
\begin{align*}
    &\norm{\frac{\bw(t)}{\norm{\bw(t)}}-\bw^\star}=\sqrt{2\left(1-\<\frac{\bw(t)}{\norm{\bw(t)}},\be_1\>\right)}
    =\sqrt{2\left(1-\frac{w_1(t)}{\sqrt{w_1^2(t)+w_2^2(t)}}\right)}
    \\=&
    \sqrt{2\left(1-\frac{1}{\sqrt{\frac{w_2^2(t)}{w_1^2(t)}+1}}\right)}=\Theta\left(\frac{w_2(t)}{w_1(t)}\right)=e^{-\Theta(t)}.
\end{align*}
From the definition of margin, we have:
\begin{align*}
    &\gamma(\bw(t))-\gamma^\star=\min_{i\in[2]}\left<\frac{\bw(t)}{\norm{\bw(t)}},\bz_i\right>-\gamma^\star=\left<\frac{\bw(t)}{\norm{\bw(t)}},\bz_2\right>-\gamma^\star
    \\=&\frac{w_1(t)\gamma^\star-w_2(t)\sqrt{1-{\gamma^\star}^2}}{\sqrt{w_1^2(t)+w_2^2(t)}}-\gamma^\star
    =\frac{-\frac{w_2(t)}{w_1(t)}\sqrt{1-{\gamma^\star}^2}+\gamma^\star}{\sqrt{\frac{w_2^2(t)}{w_1^2(t)}+1}}-\gamma^\star
    \\=&-\frac{\frac{w_2(t)}{w_1(t)}\sqrt{1-{\gamma^\star}^2}}{\sqrt{\frac{w_2^2(t)}{w_1^2(t)}+1}}+\gamma^\star\left(\left(\frac{w_2^2(t)}{w_1^2(t)}+1\right)^{-1/2}-1\right)
    =-\Theta\left(\frac{w_2(t)}{w_1(t)}\right)-\Theta\left(\frac{w_2^2(t)}{w_1^2(t)}\right)
    \\=&-\Theta\left(\frac{w_2(t)}{w_1(t)}\right)=-e^{-\Theta(t)}.
\end{align*}

\end{proof}

\subsection{Useful Lemmas}

\begin{lemma}[Margin error and Directional error]\label{lemma: Margin error and Directional error}
    Under Assumption~\ref{ass: linearly separable}, for any $\bw\in\bbR^d$, it holds that $\gamma^\star-\gamma(\bw)\leq\norm{\frac{\bw}{\norm{\bw}}-\bw^\star}$.
\end{lemma}

\begin{proof}[Proof of Lemma~\ref{lemma: Margin error and Directional error}]\ \\
Let $\bw\in\bbR^d$ and denote $i_0\in\mathop{\arg\min}\limits_{i\in[n]} y_i\<\frac{\bw}{\norm{\bw}},\bx_i\>$. Then we have:
\begin{align*}
    &\gamma^\star-\gamma(\bw)=\min_{i}y_i\<\bw^\star,\bx_i\>-\min_{i}y_i\<\frac{\bw}{\norm{\bw}},\bx_i\>
    \\=&
    \min_{i}y_i\<\bw^\star,\bx_i\>-y_{i_0}\<\frac{\bw}{\norm{\bw}},\bx_{i_0}\>
    \\\leq&
    y_{i_0}\<\bw^\star,\bx_{i_0}\>-y_{i_0}\<\frac{\bw}{\norm{\bw}},\bx_{i_0}\>=y_{i_0}\<\bw^\star-\frac{\bw}{\norm{\bw}},\bx_{i_0}\>
    \\\leq&
    \norm{\frac{\bw}{\norm{\bw}}-\bw^\star}.
\end{align*}

\begin{lemma}[Integration of \citep{soudry2018implicit,ji2020gradient}]\label{lemma: ji regularization path convergence}
For problem~\eqref{pro: logistic regression}, Gradient Flow convergences to the $\ell_2$ max-margin direction $\bw^*$, hence the regularization path also convergences to the $\ell_2$ max-margin solution: $\lim\limits_{B\to\infty}\frac{\bw_{\rm reg}^*(B)}{B}=\bw^*$.
\end{lemma}



\end{proof}


\vspace{1.cm}
\section{Proofs in Section~\ref{section: landscape analysis}}\label{appendix: Centripetal Velocity Analysis}

\begin{proof}[Proof of Theorem~\ref{thm: Centripetal Velocity Analysis, main result}]\ \\
Without loss of generality, we can assume ${\rm span}\{\bx_1,\cdots,\bx_n\}=\bbR^d$. If ${\rm span}\{\bx_1,\cdots,\bx_n\}\ne\bbR^d$, we only need to change the proof in the subspace ${\rm span}\{\bx_1,\cdots,\bx_n\}$.

Recall the definition of $\bbC(D_1,D_2;H)$:
\begin{align*}
    \bbC(D_1,D_2;H):=\left\{\bw\in{\rm span}\{\bx_i:i\in[n]\}: D_1\leq\norm{\cP_{\perp}(\bw)}\leq D_2;\left<\bw,\bw^\star\right>\geq H\right\},
\end{align*}

we further define $\bbC(D;H)$ as:
\begin{align*}
   \bbC(D;H):=\left\{\bw\in{\rm span}\{\bx_i:i\in[n]\}:\norm{\cP_{\perp}(\bw)}=D;\left<\bw,\bw^\star\right>\geq H\right\}.
\end{align*}

It holds that
\begin{align*}
    \bbC(D;H)=\left\{h\bw^\star+D\bv:h\geq H,\bv\perp\bw^\star,\norm{\bv}=1\right\},
\end{align*}
and the following relationship is easy to verified: 
\begin{align*}
    \bbC(D_1,D_2;H)=\bigcup\limits_{D_1\leq D\leq D_2}\bbC(D;H).
\end{align*}
In the following steps, we first prove the lower bound for $\bbC(D;H)$, and then prove for $\bbC(D_1,D_2;H)$.

\underline{Step I. Strip out the important ingredients.}


For any $\bw\in\bbC(D;H)$, we have:
\begin{equation}\label{equ: proof of thm centri velocity frac}
\begin{aligned}
    &\left<\frac{\nabla\cL(\bw)}{\cL(\bw)},\frac{\cP_{\perp}(\bw)}{\norm{\cP_{\perp}(\bw)}}\right>=\<\frac{\nabla\cL(\bw)}{\cL(\bw)},\bv\>
    \\=&
    \<\frac{\frac{1}{n}\sum_{i=1}^n(-y_i\bx_i)\exp(-\<\bw,y_i\bx_i\>)}{\frac{1}{n}\sum_{i=1}^n\exp(-y_i\<\bw,\bx_i\>)},\bv\>
    \\=&
    \frac{\sum_{i=1}^n\<\bv,-y_i\bx_i\>\exp\left(-h\<\bw^\star,y_i\bx_i\>\right)\exp\left(-D\<\bv,y_i\bx_i\>\right)}{\sum_{i=1}^n\exp\left(-h\<\bw^\star,y_i\bx_i\>\right)\exp\left(-D\<\bv,y_i\bx_i\>\right)}.
\end{aligned}
\end{equation}
For the numerator of~\eqref{equ: proof of thm centri velocity frac}, it holds that
\begin{align*}
    &\sum_{i=1}^n\<\bv,-y_i\bx_i\>\exp\left(-h\<\bw^\star,y_i\bx_i\>\right)\exp\left(-D\<\bv,y_i\bx_i\>\right)
    \\=&\sum_{i\in\cI}\<\bv,-y_i\bx_i\>\exp\left(-h\<\bw^\star,y_i\bx_i\>\right)\exp\left(-D\<\bv,y_i\bx_i\>\right)
    \\&+\sum_{i\notin\cI}\<\bv,-y_i\bx_i\>\exp\left(-h\<\bw^\star,y_i\bx_i\>\right)\exp\left(-D\<\bv,y_i\bx_i\>\right)
    \\=&\exp\left(-h\gamma^\star\right)\sum_{i\in\cI}\<\bv,-y_i\bx_i\>\exp\left(-D\<\bv,y_i\bx_i\>\right)
    \\&+\sum_{i\notin\cI}\<\bv,-y_i\bx_i\>\exp\left(-h\<\bw^\star,y_i\bx_i\>\right)\exp\left(-D\<\bv,y_i\bx_i\>\right)
    \\\geq&\exp\left(-h\gamma^\star\right)\sum_{i\in\cI}\<\bv,-y_i\bx_i\>\exp\left(-D\<\bv,y_i\bx_i\>\right)-\sum_{i\notin\cI}\exp\left(-h\<\bw^\star,y_i\bx_i\>\right)\exp\left(D\right)
    \\\geq&\exp\left(-h\gamma^\star\right)\sum_{i\in\cI}\<\bv,-y_i\bx_i\>\exp\left(-D\<\bv,y_i\bx_i\>\right)-(n-|\cI|)\exp\left(-h\gamma_{\rm sub}^\star\right)\exp\left(D\right);
\end{align*}
For the denominator of~\eqref{equ: proof of thm centri velocity frac}, it holds that:
\begin{align*}
    &\sum_{i=1}^n\exp\left(-h\<\bw^\star,y_i\bx_i\>\right)\exp\left(-D\<\bv,y_i\bx_i\>\right)
    \\\leq&\sum_{i=1}^n\exp\left(-h\<\bw^\star,y_i\bx_i\>\right)\exp\left(D\right)
    \\=&\Bigg(\sum_{i\in\cI}\exp\left(-h\<\bw^\star,y_i\bx_i\>\right)+\sum_{i\notin\cI}\exp\left(-h\<\bw^\star,y_i\bx_i\>\right)\Bigg)\exp\left(D\right)
    \\\leq&
    \Bigg(|\cI|\exp\left(-h\gamma^\star\right)+(n-|\cI|)\exp\left(-h\gamma_{\rm sub}^\star\right)\Bigg)\exp(D).
\end{align*}

Combining these two estimates, we obtain:
\begin{equation}\label{equ: proof of thm centri velocity frac estimate}
    \begin{aligned}
        &\left<\frac{\nabla\cL(\bw)}{\cL(\bw)},\frac{\cP_{\perp}(\bw)}{\norm{\cP_{\perp}(\bw)}}\right>
    \\\geq&\frac{\exp\left(-h\gamma^\star\right)\sum\limits_{i\in\cI}\<\bv,-y_i\bx_i\>\exp\left(-D\<\bv,y_i\bx_i\>\right)-(n-|\cI|)\exp\left(-h\gamma_{\rm sub}^\star\right)\exp\left(D\right)}{\left(|\cI|\exp\left(-h\gamma^\star\right)+(n-|\cI|)\exp\left(-h\gamma_{\rm sub}^\star\right)\right)\exp(D)}
    \\=&\frac{\frac{1}{|\cI|\exp(D)}\sum\limits_{i\in\cI}\<\bv,-y_i\bx_i\>\exp\left(-D\<\bv,y_i\bx_i\>\right)-\frac{n-|\cI|}{|\cI|}\exp\left(-h(\gamma_{\rm sub}^\star-\gamma^\star)\right)}{1+\frac{n-|\cI|}{|\cI|}\exp\left(-h(\gamma_{\rm sub}^\star-\gamma^\star)\right)}.
    \end{aligned}
\end{equation}

Notice that the term $\frac{n-|\cI|}{|\cI|}\exp\left(-h(\gamma_{\rm sub}^\star-\gamma^\star)\right)$ converges to $0$ when $h$ goes to $+\infty$. Thus, we only need to derive the uniform lower bound for the term 
\begin{align*}
    \sum_{i\in\cI}\<\bv,-y_i\bx_i\>\exp\left(-D\<\bv,y_i\bx_i\>\right)
\end{align*} 
for any $\bv\in\{\bv:\bv\perp\bw^\star,\norm{\bv}=1\}$.

\underline{Step II. Uniform Lower bound of  $\sum_{i\in\cI}\<\bv,-y_i\bx_i\>\exp\left(D\<\bv,-y_i\bx_i\>\right)$ for $\{\bv:\bv\perp\bw^\star,\norm{\bv}=1\}$.}

For simplicity, we denote $\bu_i:=\cP_\perp(-y_i\bx_i)$ for $i\in[n]$. It is worth noticing that $\<\bv,-y_i\bx_i\>=\<\bv,\bu_i\>$ for any $\bv\in\{\bv:\bv\perp\bw^\star,\norm{\bv}=1\}$. Therefore,

\begin{align*}
    \sum_{i\in\cI}\<\bv,-y_i\bx_i\>\exp\left(D\<\bv,-y_i\bx_i\>\right)=\sum_{i\in\cI}\<\bv,\bu_i\>\exp\left(D\<\bv,\bu_i\>\right).
\end{align*}

First, recalling Assumption~\ref{ass: non-degenerate data}, there exist $\alpha_i>0$ $(i\in\cI)$ such that $\bw^\star=\sum_{i\in\cI}\alpha_i y_i \bx_i$, where $\sum_{i\in\cI}\alpha_i=1$. This implies: $\bzero=\sum_{i\in\cI}\alpha_i\bu_i$. Thus, we define $\bk_i:=\alpha_i\bu_i$, then
\begin{align*}
    \sum_{i\in\cI}\bk_i=\bzero.
\end{align*}

Recalling Assumption~\ref{ass: non-degenerate data}, it holds ${\rm span}\{\bx_i:i\in\cI\}={\rm span}\{\bx_i:i\in[n]\}$, which implies that ${\rm span}\{\bu_i:i\in\cI\}=\{\bv:\bv\perp\bw^\star\}$. Therefore, there exists an absolute constant $\lambda_{\min}>0$ such that
\begin{align*}
    \sum_{i\in\cI}\<\bv,\bk_i\>^2=\bv^\top\left(\sum_{i\in\cI}\bk_i\bk_i^\top\right)\bv\geq\lambda_{\min},\quad\forall \bv\in\{\bv\perp\bw^\star,\norm{\bv}=1\}.
\end{align*}

By a rough estimate, it holds that:
\begin{align*}
    &\sum_{i\in\cI}\left|\<\bv,\bk_i\>\right|=\left(\left(\sum_{i\in\cI}\left|\<\bv,\bk_i\>\right|\right)^2\right)^{1/2}
    \\\geq&\left(\sum_{i\in\cI}\left|\<\bv,\bk_i\>\right|^2\right)^{1/2}\geq\sqrt{\lambda_{\min}},\quad\forall \bv\in\{\bv\perp\bw^\star,\norm{\bv}=1\}.
\end{align*}

Notice that $\sum_{i\in\cI}\<\bv,\bk_i\>=\<\bv,\sum_{i\in\cI}\bk_i\>=0$, it holds that
\begin{align*}
    \sum_{i\in\cI,\<\bv,\bk_i\>>0}\<\bv,\bk_i\>=-\sum_{i\in\cI,\<\bv,\bk_i\><0}\<\bv,\bk_i\>,
\end{align*}
which means:
\begin{align*}
    \sum_{i\in\cI,\<\bv,\bk_i\>>0}\<\bv,\bk_i\>=\frac{1}{2}\sum_{i\in\cI}|\<\bv,\bk_i\>|.
\end{align*}

Consequently, we can do the following estimate for the largest $\<\bv,\bk_i\>$:
\begin{align*}
    &\max\limits_{i\in\cI}\<\bv,\bk_i\>
    \geq\frac{1}{|\{i:i\in\cI,\<\bv,\bk_i\>>0\}|}\sum_{i\in\cI,\<\bv,\bk_i\>>0}\<\bv,\bk_i\>
    \\\geq&\frac{1}{|\cI|}\sum_{i\in\cI,\<\bv,\bk_i\>>0}\<\bv,\bk_i\>=\frac{1}{2|\cI|}\sum_{i\in\cI}|\<\bv,\bk_i\>|\geq\frac{\sqrt{\lambda_{\min}}}{2|\cI|},\quad \forall \bv\in\{\bv\perp\bw^\star,\norm{\bv}=1\}.
\end{align*}
 
Hence, we obtain the uniform lower bound by the following splitting: 
\begin{align*}
    &\sum_{i\in\cI}\<\bv,\bu_i\>\exp\left(D\<\bv,\bu_i\>\right)
    \\=&\sum_{i\in\cI:\<\bv,\bu_i\>>0}\<\bv,\bu_i\>\exp\left(D\<\bv,\bu_i\>\right)+\sum_{i\in\cI:\<\bv,\bu_i\><0}\<\bv,\bu_i\>\exp\left(D\<\bv,\bu_i\>\right)
    \\\geq&\max_{i\in\cI}\<\bv,\bu_i\>\exp\left(D\<\bv,\bu_i\>\right)+\sum_{i\in\cI:\<\bv,\bu_i\><0}\<\bv,\bu_i\>\exp\left(D\<\bv,\bu_i\>\right)
    \\\geq&\max_{i\in\cI}\<\bv,\bu_i\>\exp\left(D\<\bv,\bu_i\>\right)+\sum_{i\in\cI:\<\bv,\bu_i\><0}\<\bv,\bu_i\>\exp\left(0\right)
    \\\geq&\max_{i\in\cI}\<\bv,\bu_i\>\exp\left(D\<\bv,\bu_i\>\right)-\sum_{i\in\cI:\<\bv,\bu_i\><0}\norm{\bu_i}
    \\\geq&\max_{i\in\cI}\<\bv,\bu_i\>\exp\left(D\<\bv,\bu_i\>\right)-|\cI|
    \\\geq&\frac{\sqrt{\lambda_{\min}}}{2|\cI|}\exp\left(D\frac{\sqrt{\lambda_{\min}}}{2|\cI|}\right)-|\cI|,\quad \forall \bv\in\{\bv\perp\bw^\star,\norm{\bv}=1\}.
\end{align*}

\underline{Step III. The final bound.}

First, we select
\begin{align*}
    D_0=\log\left(\frac{4|\cI|^2}{\sqrt{\lambda_{\min}}}\right),
    \quad
    H_0=\frac{1}{\gamma_{\rm sub}^\star-\gamma^\star}\log\left(\max\left\{\frac{2(n-|\cI|)}{|\cI|},2\right\}\right).
\end{align*}

For any $D\geq D_0$ and $h\geq H_0$, we have:
it holds that:
\begin{align*}
    &\inf\limits_{\bv\in\{\bv\perp\bw^\star,\norm{\bv}=1\}}\sum_{i\in\cI}\<\bv,\bu_i\>\exp\left(D\<\bv,\bu_i\>\right)
    \\\geq&\frac{\sqrt{\lambda_{\min}}}{2|\cI|}\exp\left(D\frac{\sqrt{\lambda_{\min}}}{2|\cI|}\right)-|\cI|
    \geq\frac{\sqrt{\lambda_{\min}}}{4|\cI|}\exp\left(D\frac{\sqrt{\lambda_{\min}}}{2|\cI|}\right);
\end{align*}
\begin{align*}
    \frac{n-|\cI|}{|\cI|}\exp\left(-h(\gamma_{\rm sub}^\star-\gamma^\star)\right)\leq \frac{1}{2}.
\end{align*}

Therefore,
\begin{align*}
    \eqref{equ: proof of thm centri velocity frac estimate}\geq&\frac{\frac{1}{|\cI|\exp(D)}\frac{\sqrt{\lambda_{\min}}}{4|\cI|}\exp\left(D\frac{\sqrt{\lambda_{\min}}}{2|\cI|}\right)-\frac{n-|\cI|}{|\cI|}\exp\left(-h(\gamma_{\rm sub}^\star-\gamma^\star)\right)}{1+\frac{1}{2}}
    \\\geq&
    \frac{\frac{\sqrt{\lambda_{\min}}}{4|\cI|^2}\exp(-D)-\frac{n-|\cI|}{|\cI|}\exp\left(-h(\gamma_{\rm sub}^\star-\gamma^\star)\right)}{2}.
\end{align*}

Now we select
\begin{align*}
    D_1=D_0,\quad D_2=2D_0,\quad H=\max\left\{H_0,\frac{1}{\gamma_{\rm sub}^\star-\gamma^\star}\left(2D_0+\log\bracket{\frac{4|\cI|(n-|\cI|)}{\sqrt{\lambda_{\min}}}}\right)\right\}.
\end{align*}

Thus, for any $D\in[D_1,D_2]$ and $\bw\in\bbC(D;H)$, it holds that
\begin{align*}
&\left<\frac{\nabla\cL(\bw)}{\cL(\bw)},\frac{\cP_{\perp}(\bw)}{\norm{\cP_{\perp}(\bw)}}\right>
\\\geq&\frac{\frac{\sqrt{\lambda_{\min}}}{4|\cI|^2}\exp(-D)-\frac{n-|\cI|}{|\cI|}\exp\left(-h(\gamma_{\rm sub}^\star-\gamma^\star)\right)}{2}
\\\geq&\frac{\sqrt{\lambda_{\min}}\exp(-D)}{16|\cI|^2}\geq\frac{\sqrt{\lambda_{\min}}\exp(-2D_0)}{16|\cI|^2}.
\end{align*}

Lastly, we obtain our result:
\begin{align*}
    &\inf\limits_{\bw\in\bbC(D_1,D_2;H)}\left<\frac{\nabla\cL(\bw)}{\cL(\bw)},\frac{\cP_{\perp}(\bw)}{\norm{\cP_{\perp}(\bw)}}\right>
    \\=&\inf_{D\in[D_1, D_2]}\inf\limits_{\bw\in\bbC(D;H)}\left<\frac{\nabla\cL(\bw)}{\cL(\bw)},\frac{\cP_{\perp}(\bw)}{\norm{\cP_{\perp}(\bw)}}\right>\geq\frac{\sqrt{\lambda_{\min}}\exp(-D)}{16|\cI|^2}
    \\&\geq\frac{\sqrt{\lambda_{\min}}\exp(-2D_0)}{16|\cI|^2}>0.
\end{align*}

\end{proof}

\vspace{1.cm}
\section{Proofs in Section~\ref{section: convergence analysis}}\label{appendix: proof: Convergence Analysis}

\subsection{Proof of Theorem~\ref{thm: PRGD main thm}}
\label{appendix: proof: PPGD}

\begin{proof}[Proof of Theorem~\ref{thm: PRGD main thm}]\ \\
According Theorem~\ref{thm: Centripetal Velocity Analysis, main result}, there exist constants $H,D,\mu>0$ such that 
\begin{align*}
    \<\frac{\nabla\cL(\bw)}{\cL(\bw)},\frac{\cP_{\perp}\left(\bw\right)}{\norm{\cP_{\perp}\left(\bw\right)}}\>\geq\mu\text{ holds for any }\bw\in\bbC(D,2D;H),
\end{align*}
where $\bbC(D,2D;H):=\left\{\bw\in{\rm span}\{\bx_i:i\in[n]\}: D\leq\norm{\cP_{\perp}(\bw)}\leq 2D;\left<\bw,\bw^\star\right>\geq H\right\}$.

Following the proof of Theorem~\ref{thm: Centripetal Velocity Analysis, main result}, we further define $\bbC(D;H)$ as:
\begin{align*}
   \bbC(D;H):=\left\{\bw\in{\rm span}\{\bx_i:i\in[n]\}:\norm{\cP_{\perp}(\bw)}=D;\left<\bw,\bw^\star\right>\geq H\right\}.
\end{align*}

It holds that $\bbC(D;H)=\left\{h\bw^\star+D\bv:h\geq H,\bv\perp\bw^\star,\norm{\bv}=1\right\}$ and $\bbC(D;H)\subset\bbC(D,2D;H)$.

\underline{Analysis of Phase I.} 

Phase I is a warm-up phase. We will prove that at the end of this phase, trained $\bw$ can be scaled onto $\bbC(D;H)$. First, we choose the error
\begin{align*}
    \epsilon=\min\left\{\frac{D}{2H},\frac{1}{2}\right\}.
\end{align*}

Notice that Theorem~\ref{thm: NGD upper bound ji} ensures that under Assumption~\ref{ass: linearly separable} and~\ref{ass: non-degenerate data}, the directional convergence of GD and NGD with $\eta\leq1$ holds, and the rates are $\cO(1/\log t)$ and $\cO(1/t)$,respectively.

Therefore, there exists $T^{\epsilon}$ such that $\norm{\frac{\bw(T^{\epsilon})}{\norm{\bw(T^{\epsilon})}}-\bw^\star}<\epsilon$, which implies the inner satisfies:
\begin{align*}
    \left<\frac{\bw(T^{\epsilon})}{\norm{\bw(T^{\epsilon})}},\bw^\star\right>=\frac{1}{2}\left(2-\norm{\frac{\bw(T^{\epsilon})}{\norm{\bw(T^{\epsilon})}}-\bw^\star}^2\right)>1-\frac{\epsilon^2}{2}.
\end{align*}
Therefore, at $T^{\epsilon}$, it holds that:
\begin{align*}
    &\frac{\norm{\cP_\perp(\bw(T^{\epsilon}))}}{\left<\bw(T^{\epsilon}),\bw^\star\right>}
    =\frac{\norm{\bw(T^{\epsilon})-\cP(\bw(T^{\epsilon}))}}{\left<\bw(T^{\epsilon}),\bw^\star\right>}=\norm{\frac{\bw(T^{\epsilon})}{\left<\bw(T^{\epsilon}),\bw^\star\right>}-\bw^\star}
    \\=&
    \norm{\frac{\bw(T^{\epsilon})}{\left<\bw(T^{\epsilon})),\bw^\star\right>}-\bw^\star}=\norm{\frac{\bw(T^{\epsilon})}{\left<\bw(T^{\epsilon})),\bw^\star\right>}-\frac{\bw(T^{\epsilon})}{\norm{\bw(T^{\epsilon})}}+\frac{\bw(T^{\epsilon})}{\norm{\bw(T^{\epsilon})}}-\bw^\star}
    \\\leq&
    \norm{\frac{\bw(T^{\epsilon})}{\left<\bw(T^{\epsilon})),\bw^\star\right>}-\frac{\bw(T^{\epsilon})}{\norm{\bw(T^{\epsilon})}}}+\norm{\frac{\bw(T^{\epsilon})}{\norm{\bw(T^{\epsilon})}}-\bw^\star}
    <\left|\frac{\left<\bw(T^{\epsilon})),\bw^\star\right>-\norm{\bw(T^{\epsilon})}}{\left<\bw(T^{\epsilon})),\bw^\star\right>}\right|+\epsilon
    \\=&
    \left|1-\frac{1}{\left<\frac{\bw(T^{\epsilon})}{\norm{\bw(T^{\epsilon})}},\bw^\star\right>}\right|+\epsilon<\frac{1}{1-\frac{\epsilon^2}{2}}-1+\epsilon=\frac{\frac{\epsilon^2}{2}}{1-\frac{\epsilon^2}{2}}+\epsilon<\frac{4\epsilon^2}{7}+\epsilon
    \\\leq&
    \left(\frac{2}{7}+1\right)\epsilon\leq 2\epsilon\leq\min\left\{\frac{D}{H},1\right\}.
\end{align*}
We choose $T_{\rm w}=T^{\epsilon}=\Theta(1)$, and we obtain $\bw(T_{\rm w})$ at the end of Phase I.

\underline{Analysis of Phase II.}

For simplicity, due to $T_{\rm w}$ is an constant, we replace the time $t$ to $t-T_{\rm w}$ in the proof of Phase II. This means that Phase II starts from $t=0$ with the initialization $\bw(0)\gets\bw(T_{\rm w})$.

In this proof, we choose
\begin{align*}
    \eta=\mu D,\ T_{k}=2k,\ R_k=\frac{D\norm{\bw(T_k)}}{\norm{\cP_{\perp}(\bw(T_k))}},\ \forall k\geq0.
\end{align*}

Recalling Algorithm~\ref{alg: PRGD}, the update rule is:
\begin{align*}
\cdots;
\\
\bw(2k+1)&=R_{k}\frac{\bw(2k)}{\norm{\bw(2k)}};
\\
\bv(2k+2)&=\bw(2k+1)-\eta\frac{\nabla\cL(\bw(2k+1))}{\cL(\bw(2k+1))};
\\
\bw(2k+2)&=\Proj_{\bbB(0,\norm{\bw(2k+1)})}\left(\bv(2k+2)\right);
\\
\bw(2k+3)&=R_{k+1}\frac{\bw(2k+2)}{\norm{\bw(2k+2)}};
\\\cdots
\end{align*}

In general, we aim to prove the following statements:
\begin{align*}
        &\text{(S1). $\bw(2k+1)\in\bbC(D;H), \ \forall k\geq0$.}
        \\
        &\text{(S2). $\<\bw(2k+1),\bw^\star\>\geq\frac{1}{\left(\sqrt{1-2\mu}\right)^{k}}\left(\<\bw(1),\bw^\star\>+\frac{\gamma^\star}{1-\sqrt{1-2\mu}}\right)-\frac{\gamma^\star}{1-\sqrt{1-2\mu}},\ \forall k\geq0$;}
        \\&\quad\quad
        \<\bw(2k+1),\bw^\star\>\leq\frac{1}{\left(\sqrt{1-\mu^2}\right)^{k}}\left(\<\bw(1),\bw^\star\>+\frac{1}{1-\sqrt{1-\mu^2}}\right)-\frac{1}{1-\sqrt{1-\mu^2}},\ \forall k\geq0.
        \\&
        \text{(S3). $D\sqrt{1-2\mu}\leq\norm{\cP_{\perp}(\bv(2k+2))}\leq D\sqrt{1-\mu^2},\ \forall k\geq0$.}
        \\&
        \text{(S4). }\frac{\<\bw(2k+2),\bw^\star\>}{\norm{\cP_{\perp}(\bw(2k+2))}}=\frac{\<\bw(2k+3),\bw^\star\>}{\norm{\cP_{\perp}(\bw(2k+3))}}=\frac{\<\bw(1),\bw^\star\>}{D}e^{\Theta(k)}.
        \\&
        \text{(S5). } R_{k+1}=\<\bw(1),\bw^\star\>e^{\Theta(k)}.
        \\&
        \text{(S6). } \norm{\frac{\bw(t)}{\norm{\bw(t)}}-\bw^\star}=\frac{D}{\<\bw(1),\bw^\star\>}e^{-\Theta(t)}.
        \\&
        \text{(S7). } \gamma^\star-\gamma(\bw(t))=\frac{D}{\<\bw(1),\bw^\star\>}e^{-\Theta(t)}.
\end{align*}

\underline{Step I. Proof of (S1)(S2).} 

In this step, we will prove (S1)(S2) by induction.

\underline{Step I (i).} 
We prove (S1)(S2) for $k=0$. Recalling our analysis of Phase I, it holds that 
\begin{align*}
    \frac{\norm{\cP_\perp(\bw(0))}}{\<\bw(0)),\bw^\star\>}\leq\min\left\{\frac{D}{H},1\right\}.
\end{align*}
Thus, if we choose $R_0=\frac{D\norm{\bw(0)}}{\norm{\cP_{\perp}(\bw(0))}}$ in Algorithm~\ref{alg: PRGD}, then $\bw(1)=\frac{D}{\norm{\cP_{\perp}(\bw(0))}}\cdot\bw(0)$ and $\bw(1)$ satisfies:
\begin{gather*}
    \norm{\cP_{\perp}(\bw(1))}=\norm{\cP_{\perp}\left(\frac{D}{\norm{\cP_{\perp}(\bw(0))}}\bw(0)\right)}=\norm{\frac{D\cP_{\perp}(\bw(0))}{\norm{\cP_{\perp}(\bw(0))}}}=D;
\end{gather*}
\begin{align*}
    &\left<\bw(1),\bw^\star\right>=\left<\frac{D}{\norm{\cP_{\perp}(\bw(0))}}\bw(0),\bw^\star\right>
    =D\left<\frac{\bw(0)}{\norm{\cP_\perp(\bw(0))}},\bw^\star\right>
    \\=&D\frac{\<\bw(0),\bw^\star\>}{\norm{\cP_\perp(\bw(0))}}\geq \frac{D}{\min\left\{\frac{D}{H},1\right\}}=\max\left\{H,D\right\}.
\end{align*}
which means that (S1) $\bw(1)\in\bbC(D;H)$ holds for $k=0$. As for (S2), it is trivial for $k=0$.

\underline{Step I (ii).}
Assume (S1)(S2) hold for any $0\leq k'\leq k$. Then we will prove for $k'=k+1$.

First, it is easy to bound the difference between $\left<\bv(2k+2),\bw^\star\right>$ and $\left<\bw(2k+1),\bw^\star\right>$:
\begin{equation}\label{equ: proof of thm: main PRGD: diff between v(2k+2) and w(2k+1)}
\begin{aligned}
    &\left<\bv(2k+2),\bw^\star\right>-\left<\bw(2k+1),\bw^\star\right>
    \\=&
    \eta\left<-\frac{\nabla\cL(\bw(2k+1))}{\cL(\bw(2k+1))},\bw^\star\right>
    \overset{\text{Lemma~\ref{lemma: uniform project lower bound}}}{\in}\left[\eta\gamma^\star,\eta\right]=[\mu\gamma^\star D,\mu D].
\end{aligned}
\end{equation}
Secondly, notice the following fact about $\bw(2k+3)$:
\begin{equation}\label{equ: proof of thm: main PRGD: fact w 2k+3}
    \begin{aligned}
        &\bw(2k+3)=R_{k+1}\frac{\bw(2k+2)}{\norm{\bw(2k+2)}}=\frac{D\norm{\bw(2k+2)}}{\norm{\cP_{\perp}(\bw(2k+2))}}\frac{\bw(2k+2)}{\norm{\bw(2k+2)}}
    \\=&
    \frac{D}{\norm{\cP_\perp(\bw(2k+2))}}\bw(2k+2)=\frac{D}{\norm{\cP_\perp(\bv(2k+2))}}\bv(2k+2).
    \end{aligned}
\end{equation}
With the help of the estimates above and the induction, now we can give the following two-sided bounds for $\<\bw(2k+3),\bw^\star\>$.
\begin{itemize}
    \item Upper bound for $\<\bw(2k+3),\bw^\star\>$:
\end{itemize}
\begin{equation}\label{equ: proof of thm: main PRGD: upper bound w 2k+3}
\begin{aligned}
    &\<\bw(2k+3),\bw^\star\>
    \overset{\eqref{equ: proof of thm: main PRGD: fact w 2k+3}}{=}
    D\frac{\left<\bv(2k+2),\bw^\star\right>}{\norm{\cP_\perp(\bv(2k+2))}}
    \\\overset{\eqref{equ: proof of thm: main PRGD: diff between v(2k+2) and w(2k+1)}}{\leq}& 
    D\frac{\left<\bw(2k+1),\bw^\star\right>+1}{\sqrt{1-\mu^2}D}
    =\frac{\left<\bw(2k+1),\bw^\star\right>+1}{\sqrt{1-\mu^2}}
    \\
    \overset{\text{induction}}{\leq}&
    \frac{1}{\sqrt{1-\mu^2}}\left(\frac{1}{\left(\sqrt{1-\mu^2}\right)^{k}}\left(\<\bw(1),\bw^\star\>+\frac{1}{1-\sqrt{1-\mu^2}}\right)-\frac{1}{1-\sqrt{1-\mu^2}}+1\right)
    \\=&
    \frac{1}{\left(\sqrt{1-\mu^2}\right)^{k+1}}\left(\<\bw(1),\bw^\star\>+\frac{1}{1-\sqrt{1-\mu^2}}\right)-\frac{1}{1-\sqrt{1-\mu^2}}.
    \end{aligned}
\end{equation}
\begin{itemize}
    \item Lower bound for $\<\bw(2k+3),\bw^\star\>$:
\end{itemize}
\begin{equation}\label{equ: proof of thm: main PRGD: lower bound w 2k+3}
    \begin{aligned}
    &\<\bw(2k+3),\bw^\star\>
    \overset{\eqref{equ: proof of thm: main PRGD: fact w 2k+3}}{=}
    D\frac{\left<\bv(2k+2),\bw^\star\right>}{\norm{\cP_\perp(\bv(2k+2))}}
    \\
    \overset{\eqref{equ: proof of thm: main PRGD: diff between v(2k+2) and w(2k+1)}}{\geq}& D\frac{\left<\bw(2k+1),\bw^\star\right>+\gamma^\star}{\sqrt{1-2\mu}D}
    =\frac{\left<\bw(2k+1),\bw^\star\right>+\gamma^\star}{\sqrt{1-2\mu}}
    \\\overset{\text{induction}}{\geq}&\frac{1}{\sqrt{1-2\mu}}\left(\frac{1}{\left(\sqrt{1-2\mu}\right)^{k}}\left(\<\bw(1),\bw^\star\>+\frac{\gamma^\star}{1-\sqrt{1-2\mu}}\right)-\frac{\gamma^\star}{1-\sqrt{1-2\mu}}+\gamma^\star\right)
    \\=&\frac{1}{\left(\sqrt{1-2\mu}\right)^{k+1}}\left(\<\bw(1),\bw^\star\>+\frac{\gamma^\star}{1-\sqrt{1-2\mu}}\right)-\frac{\gamma^\star}{1-\sqrt{1-2\mu}}.
    \end{aligned}
\end{equation}

Hence, from~\eqref{equ: proof of thm: main PRGD: upper bound w 2k+3}\eqref{equ: proof of thm: main PRGD: lower bound w 2k+3}, we have proved that (S2) holds for $k+1$.

Moreover, we have the following facts:
\begin{align*}
    \norm{\cP_{\perp}\left(\bw(2k+3)\right)}\overset{\eqref{equ: proof of thm: main PRGD: fact w 2k+3}}{=}&\norm{\cP_{\perp}\left(\frac{D}{\norm{\cP_\perp(\bv(2k+2))}}\bv(2k+2)\right)}
    \\=&\norm{D\frac{\cP_\perp(\bv(2k+2))}{\norm{\cP_\perp(\bv(2k+2))}}}=D,
\end{align*}
\begin{align*}
    \<\bw(2k+3),\bw^\star\>
    \overset{\eqref{equ: proof of thm: main PRGD: lower bound w 2k+3}}{\geq}
    &\frac{1}{\left(\sqrt{1-2\mu}\right)^{k+1}}\left(\<\bw(1),\bw^\star\>+\frac{\gamma^\star}{1-\sqrt{1-2\mu}}\right)-\frac{\gamma^\star}{1-\sqrt{1-2\mu}}
    \\\geq&
    \frac{\<\bw(1),\bw^\star\>}{\left(\sqrt{1-2\mu}\right)^{k+1}}\geq\<\bw(1),\bw^\star\>\geq H;
\end{align*}
which means that (S1) holds for $k+1$, i.e., $\bw(2k+3)\in\bbC(D;H)$.

Now we have proved (S1)(S2) for any $k\geq 0$ by induction.

\underline{Step II. Proof of (S3).} 

In this step, we will prove (S3) directly. For any $k\geq0$, we can derive the following two-sides bounds:
\begin{itemize}
    \item For the upper bound of $\norm{\cP_{\perp}(\bv(2k+2))}$, we have:
\end{itemize}
\begin{equation}\label{equ: proof of thm: main PRGD: upper bound P perp(v)}
\begin{aligned}
    &\norm{\cP_{\perp}(\bv(2k+2))}^2
    =\norm{\cP_{\perp}(\bw(2k+1))-\eta\cP_{\perp}\left(\frac{\nabla\cL(\bw(2k+1))}{\cL(\bw(2k+1))}\right)}^2
    \\=&
    \norm{\cP_{\perp}(\bw(2k+1))}^2+\eta^2\norm{\cP_{\perp}\left(\frac{\nabla\cL(\bw(2k+1))}{\cL(\bw(2k+1))}\right)}^2
    \\&\quad\quad\quad\quad-2\eta\left<\cP_{\perp}(\bw(2k+1)),\cP_{\perp}\left(\frac{\nabla\cL(\bw(2k+1))}{\cL(\bw(2k+1))}\right)\right>
    \\=&
    D^2+\eta^2\norm{\cP_{\perp}\left(\frac{\nabla\cL(\bw(2k+1))}{\cL(\bw(2k+1))}\right)}^2-2\eta D\left<\frac{\cP_{\perp}(\bw(2k+1))}{\norm{\cP_{\perp}(\bw(2k+1))}},\frac{\nabla\cL(\bw(2k+1))}{\cL(\bw(2k+1))}\right>
    \\\leq&
    D^2+\eta^2\norm{\frac{\nabla\cL(\bw(2k+1))}{\cL(\bw(2k+1))}}^2-2\eta D\left<\frac{\cP_{\perp}(\bw(2k+1))}{\norm{\cP_{\perp}(\bw(2k+1))}},\frac{\nabla\cL(\bw(2k+1))}{\cL(\bw(2k+1))}\right>
    \\\overset{\text{Lemma~\ref{lemma: uniform project lower bound}}}{\leq}&
    D^2+\eta^2-2\eta D\mu=D^2+\mu^2 D^2-2\mu^2 D^2=(1-\mu^2)D^2.
\end{aligned}
\end{equation}
\begin{itemize}
    \item For the lower bound of $\norm{\cP_{\perp}(\bv(2k+2))}$, we have:
\end{itemize}
\begin{equation}\label{equ: proof of thm: main PRGD: lower bound P perp(v)}
\begin{aligned}
    &\norm{\cP_{\perp}(\bv(2k+2))}^2
    =\norm{\cP_{\perp}(\bw(2k+1))-\eta\cP_{\perp}\left(\frac{\nabla\cL(\bw(2k+1))}{\cL(\bw(2k+1))}\right)}^2
    \\=&
    D^2+\eta^2\norm{\cP_{\perp}\left(\frac{\nabla\cL(\bw(2k+1))}{\cL(\bw(2k+1))}\right)}^2
    \\&\quad\quad\quad\quad-2\eta D\left<\frac{\cP_{\perp}(\bw(2k+1))}{\norm{\cP_{\perp}(\bw(2k+1))}},\frac{\nabla\cL(\bw(2k+1))}{\cL(\bw(2k+1))}\right>
    \\\geq&
    D^2-2\eta D\left<\frac{\cP_{\perp}(\bw(2k+1))}{\norm{\cP_{\perp}(\bw(2k+1))}},\frac{\nabla\cL(\bw(2k+1))}{\cL(\bw(2k+1))}\right>
    \\\geq&
    D^2-2\eta D\norm{\frac{\nabla\cL(\bw(2k+1))}{\cL(\bw(2k+1))}}
    \\\overset{\text{Lemma~\ref{lemma: uniform project lower bound}}}{\geq}&
    D^2-2\eta D=D^2-2\mu D^2=(1-2\mu)D^2.
\end{aligned}
\end{equation}

Hence, we have proved (S3).

\underline{Step III. Proof of (S4)(S5)(S6).}

First, we derive two-sided bounds for $\frac{\<\bw(2k+2),\bw^\star\>}{\norm{\cP_{\perp}(\bw(2k+2))}}$. For any $k\geq 0$, we have:
\begin{itemize}
    \item Upper bound of $\frac{\<\bw(2k+2),\bw^\star\>}{\norm{\cP_{\perp}(\bw(2k+2))}}$.
\end{itemize}
\begin{align*}
    &\frac{\<\bw(2k+2),\bw^\star\>}{\norm{\cP_{\perp}(\bw(2k+2))}}
    =\frac{\<\bv(2k+2),\bw^\star\>}{\norm{\cP_{\perp}(\bv(2k+2))}}
    \\\overset{\eqref{equ: proof of thm: main PRGD: diff between v(2k+2) and w(2k+1)}}{\leq}&
    \frac{\<\bw(2k+1),\bw^\star\>+\mu D}{\norm{\cP_{\perp}(\bv(2k+2))}}
    \\\overset{\text{(S2)(S3)}}{\leq}&
    \frac{\frac{1}{\left(\sqrt{1-\mu^2}\right)^{k}}\left(\<\bw(1),\bw^\star\>+\frac{1}{1-\sqrt{1-\mu^2}}\right)-\frac{1}{1-\sqrt{1-\mu^2}}+\mu D}{D\sqrt{1-2\mu}}
    \\\leq&
    \frac{\<\bw(1),\bw^\star\>}{D}e^{\Theta(k)}.
\end{align*}
\begin{itemize}
    \item Lower bound of $\frac{\<\bw(2k+2),\bw^\star\>}{\norm{\cP_{\perp}(\bw(2k+2))}}$.
\end{itemize}
\begin{align*}
    &\frac{\<\bw(2k+2),\bw^\star\>}{\norm{\cP_{\perp}(\bw(2k+2))}}
    =\frac{\<\bv(2k+2),\bw^\star\>}{\norm{\cP_{\perp}(\bv(2k+2))}}
    \\\overset{\eqref{equ: proof of thm: main PRGD: diff between v(2k+2) and w(2k+1)}}{\geq}&
    \frac{\<\bw(2k+1),\bw^\star\>+\mu\gamma^\star D}{\norm{\cP_{\perp}(\bv(2k+2))}}
    \\\overset{\text{(S2)(S3)}}{\geq}&
    \frac{\frac{1}{\left(\sqrt{1-2\mu}\right)^{k}}\left(\<\bw(1),\bw^\star\>+\frac{\gamma^\star}{1-\sqrt{1-2\mu}}\right)-\frac{\gamma^\star}{1-\sqrt{1-2\mu}}+\mu\gamma^\star D}{(1-\mu)D}
    \\\geq&
    \frac{\<\bw(1),\bw^\star\>}{D}e^{\Theta(k)}.
\end{align*}
Additionally, notice 
\begin{align*}
    \frac{\<\bw(2k+2),\bw^\star\>}{\norm{\cP_{\perp}(\bw(2k+2))}}=\frac{\<\bw(2k+3),\bw^\star\>}{\norm{\cP_{\perp}(\bw(2k+3))}},
\end{align*}
we obtain (S4):
\begin{align*}
    \frac{\<\bw(2k+2),\bw^\star\>}{\norm{\cP_{\perp}(\bw(2k+2))}}=\frac{\<\bw(2k+3),\bw^\star\>}{\norm{\cP_{\perp}(\bw(2k+3))}}=\frac{\<\bw(1),\bw^\star\>}{D}e^{\Theta(k)}.
\end{align*}
Furthermore, Combining (S4) and the following fact
\begin{align*}
    &R_{k+1}=\frac{D\norm{\bw(2k+2)}}{\norm{\cP_{\perp}(\bw(2k+2))}}=\frac{D\norm{\bv(2k+2)}}{\norm{\cP_{\perp}(\bv(2k+2))}}
    \\=&D\frac{\sqrt{\left<\bv(2k+2),\bw^\star\right>^2+\norm{\cP_{\perp}(\bv(2k+2))}^2}}{\norm{\cP_{\perp}(\bv(2k+2))}}
    =D\sqrt{\frac{\left<\bv(2k+2),\bw^\star\right>^2}{\norm{\cP_{\perp}(\bv(2k+2))}^2}+1},
\end{align*}
we can obtain (S5):
\begin{align*}
    R_{k+1}=\<\bw(1),\bw^\star\>e^{\Theta(k)}.
\end{align*}

In the same way, we can prove
\begin{align*}
    &\norm{\frac{\bw(2k)}{\norm{\bw(2k)}}-\bw^\star}=\norm{\frac{\bw(2k+1)}{\norm{\bw(2k+1)}}-\bw^\star}=2\left(1-\<\frac{\bw(2k+1)}{\norm{\bw(2k+1)}},\bw^\star\>\right)
    \\=&
    2\left(1-\frac{\<\bw(2k+1),\bw^\star\>}{\norm{\bw(2k+1)}}\right)=2\left(1-\frac{\<\bw(2k+1),\bw^\star\>}{\sqrt{\<\bw(2k+1),\bw^\star\>^2+\norm{\cP_{\perp}(\bw(2k+1))}^2}}\right)
    \\=&2\left(1-\frac{1}{\sqrt{1+\frac{\norm{\cP_{\perp}(\bw(2k+1))}^2}{\<\bw(2k+1),\bw^\star\>^2}}}\right)
    \overset{\text{(S4)}}{=}
    2\left(1-\frac{1}{\sqrt{1+\frac{D^2}{\<\bw(1),\bw^\star\>^2}e^{-\Theta(k)}}}\right)
    \\=&\frac{D}{\<\bw(1),\bw^\star\>}e^{-\Theta(k)},
\end{align*}
which means (S6):
\begin{align*}
    \norm{\frac{\bw(t)}{\norm{\bw(t)}}-\bw^\star}=\frac{D}{\<\bw(1),\bw^\star\>}e^{-\Theta(t)},
\end{align*}

\underline{Step III. Proof of (S7).}
Using Lemma~\ref{lemma: Margin error and Directional error} and (S6), we obtain (S7).

\underline{Conclusions.}

From our proof of Phase II, we have $\<\bw(1),\bw^\star\>\geq\max\{H,D\}$. Taking this fact into (S6)(S7), we obtain our conclusions:
\begin{align*}
    &\norm{\frac{\bw(t)}{\norm{\bw(t)}}-\bw^\star}=e^{-\Omega(t)};
    \\&\gamma^\star-\gamma(\bw(t))=e^{-\Omega(t)}.
\end{align*}

\end{proof}

\subsection{Proof of Theorem~\ref{thm: GD NGD main result}}
\label{appendix: proof of hardness of GD and NGD}

The proof of GD is relatively straightforward. In contrast, the proof for NGD is significantly more intricate, necessitating a more rigorous convex optimization analysis than Proposition~\ref{thm: 3data}. 

For NGD, we still focus on the dynamics of $\cP_{\perp}(\bw(t))$, which is orthogonal to $\bw^\star$.
Actually, we can prove that there exists a subsequence $\cP_{\perp}(\bw(t_k))$ ($t_k\to\infty$), which satisfies $\norm{\cP_{\perp}(\bw(t_k))}=\Theta(1)$. Since the norm grows at $\norm{\bw(t_k)}=\Theta(t_k)$ (Thm~\ref{thm: NGD upper bound ji}), NGD must have only $\Theta(\norm{\cP_{\perp}(\bw(t_k))}/\norm{\bw(t_k)})=\Theta(1/t_k)$ directional convergence rate. 
Furthermore, the non-degenerated data assumption~\ref{ass: non-degenerate data} can also provide a two-sided bound for the margin error (Lemma~\ref{lemma: tight margin error and Directional error}), which ensures $\Omega(1/t_k^2)$ margin maximization rate.
Our crucial point is that the $(d-1)$-dim dynamics of $\cP_{\perp}(\bw(t))$ near $\bzero\in{\rm span}\{\bx_i:i\in\cI\}$ are close to {\rm in-exact} gradient descent dynamics on another strongly convex loss $\cL_\perp(\cdot)$ with unique minimizer $\bv^\star\in\bbR^{d-1}$.
Moreover, our condition $\gamma^\star\bw^\star\ne\frac{1}{|\cI|}\sum_{i\in\cI}\bx_i y_i$ can ensure that $\bv^\star\ne\bzero$. Therefore, there must exists a sequence $\cP_\perp(\bw(t_k))$ which can escapes from a sufficient small ball $\bbB(\bzero_{d-1};\epsilon_0)$, which means $\norm{\cP_{\perp}(\bw(t_k))}=\Theta(1)$.

\begin{proof}[Proof for GD]\ \\
The proof for GD is straightforward.

By Theorem~\ref{thm: refined error result, soudry}, $\lim\limits_{t\to+\infty}\Big(\bw(t)-\bw^\star\log t\Big)=\tilde{\bw}$, where $\tilde{\bw}$ is the solution to the equations: 
\begin{align*}
    \eta\exp\left(-\<\tilde{\bw},\bx_i y_i\>\right)=\alpha_i,i\in\cI.
\end{align*}

\underline{Step I. $\cP_\perp(\tilde{\bw})\ne\bzero$}.

If we assume $\cP_\perp(\tilde{\bw})=\bzero$, then there exists $c>0$ such that $\tilde{\bw}=c\bw^\star$.

Notice that for any $i\in\cI$, $\<\bw^\star,\bx_i y_i\>=\gamma^\star$. Therefore, there exists $c'>0$ such that $\alpha_i=c'$ for any $i\in\cI$.
Recalling $\bw^\star=\sum_{i\in\cI}\alpha_i\bx_i y_i$, we have $\bw^\star=c'\sum_{i\in\cI}\bx_i y_i$, which implies $\bw^\star=\frac{1}{|\cI|\gamma^\star}\sum_{i\in\cI}\bx_i y_i$. This is contradict to our condition $\gamma^\star\bw^\star\ne\frac{1}{|\cI|}\sum_{i\in\cI}\bx_i y_i$.

Hence, we have proved $\cP_\perp(\tilde{\bw})\ne\bzero$.

\underline{Step II. The lower bound.}

Recalling
$\lim\limits_{t\to+\infty}\Big(\bw(t)-\bw^\star\log t\Big)=\tilde{\bw}$ and our results in Step I, there exists $T_0>0$ such that
\begin{align*}
    \norm{\bw(t)-\bw^\star\log t-\tilde{\bw}}\leq\frac{\norm{\cP_\perp
    (\tilde{\bw})}}{2},\quad\forall t\geq T_0.
\end{align*}

Using the fact $\norm{\cP_\perp(\bw)}\leq\norm{\bw}$, we have
\begin{align*}
    \norm{\cP_\perp(\bw(t))-\cP_\perp(\tilde{\bw})}\leq\norm{\bw(t)-\bw^\star\log t-\tilde{\bw}}\leq\frac{\norm{\cP_\perp
    (\tilde{\bw})}}{2},\quad\forall t\geq T_0,
\end{align*}
which implies
\begin{align*}
    \frac{\norm{\cP_\perp
    (\tilde{\bw})}}{2}\leq\norm{\cP_\perp(\bw(t))}\leq\frac{3\norm{\cP_\perp
    (\tilde{\bw})}}{2},\quad\forall t\geq T_0.
\end{align*}

Recalling Theorem~\ref{thm: NGD upper bound ji}, it holds that $\norm{\bw(t)}=\Theta(\log t)$.
Then, a direct calculation ensures that:
\begin{align*}
    &\norm{\frac{\bw(t)}{\norm{\bw(t)}}-\bw^\star}^2=2-2\frac{\<\bw(t),\bw^\star\>}{\norm{\bw(t)}}
    \\=&2-2\frac{\<\bw(t),\bw^\star\>}{\sqrt{\<\bw(t),\bw^\star\>^2+\norm{\cP_\perp(\bw(t))}^2}}
    \\=&2-\frac{2}{\sqrt{1+\frac{\norm{\cP_\perp(\bw(t))}^2}{\<\bw(t),\bw^\star\>^2}}}=\Theta\left(\frac{\norm{\cP_\perp(\bw(t))}^2}{\<\bw(t),\bw^\star\>^2}\right)
    \\=&\Theta\left(\frac{\norm{\cP_\perp(\bw(t_k))}^2}{\norm{\bw(t)}^2-\norm{\cP_\perp(\bw(t))}^2}\right)=\Theta\left(\frac{1}{\log^2 t}\right),
\end{align*}
which implies the tight bound for the directional convergence rate: $\norm{\frac{\bw(t)}{\norm{\bw(t)}}-\bw^\star}=\Theta\left(\frac{1}{\log t}\right)$.
Moreover, with the help of Lemma~\ref{lemma: tight margin error and Directional error}, we have the lower bound for the margin maximization rate: $\gamma^\star-\gamma(\bw(t))=\Omega\left(\frac{1}{\log^2 t}\right)$.

\end{proof}

\begin{proof}[Proof for NGD]\ \\
NGD is more difficult to analyze than GD due to the more aggressive step size, and we need more refined convex optimization analysis.

Without loss of generality, we can assume ${\rm span}\{\bx_1,\cdots,\bx_n\}=\bbR^d$. This is because: GD, NGD, and PRGD can only evaluate in ${\rm span}\{\bx_i:i\in[n]\}$, i.e. $\bw(t)\in {\rm span}\{\bx_i:i\in[n]\}$. If ${\rm span}\{\bx_1,\cdots,\bx_n\}\ne\bbR^d$, we only need to change the proof in the subspace ${\rm span}\{\bx_1,\cdots,\bx_n\}$. For simplicity, we still denote $\bz_i:=\bx_iy_i$ ($i\in[n]$).

With the help of Theorem~\ref{thm: NGD upper bound ji}, the upper bounds hold: $\gamma^\star-\gamma(\bw(t))\leq\norm{\frac{\bw(t)}{\norm{\bw(t)}}-\bw^\star}=\cO(1/t)$.
So we only need to prove the lower bounds for $\gamma^\star-\gamma(\bw(t))$ and $\norm{\frac{\bw(t)}{\norm{\bw(t)}}-\bw^\star}$.

\underline{Proof Outline.}

We {\bf aim to prove the following claim}:
\begin{gather*}
    \text{there exists constant $T_0>0$ and $\epsilon_0>0$, such that:}
    \\
    \text{for any $T>T_0$, there exists $t>T$ s.t. $\norm{\cP_{\perp}(\bw(t))}>\epsilon_0$.}
\end{gather*}
If we can prove this conclusion, then there exists a subsequence $\bw(t_k)$ satisfying $t_{k+1}>t_k$, $t_k\to\infty$, and $\norm{\cP_\perp(\bw(t_k))}>\epsilon_0$. Recalling Theorem~\ref{thm: NGD upper bound ji}, $\norm{\bw(t_k)}=\Theta(t_k)$. Therefore, it must holds $\norm{\frac{\bw(t_k)}{\norm{\bw(t_k)}}-\bw^\star}=\Omega(1/t_k)$.

\underline{Proof Preparation.}

For simplicity, we denote the optimization problem orthogonal to $\bw^\star$ as
\begin{align*}
    \min_{\bv}:\cL_{\perp}(\bv)=\frac{1}{|\cI|}\sum_{i\in\cI}\exp\left(-y_i\<\bv,\cP_{\perp}(\bx_i)\>\right), \bv\in{\rm span}\{\cP_{\perp}(\bx_i):i\in\cI\}.
\end{align*}

In this proof, we focus on the dynamics of $\cP_{\perp}(\bw(t))$, satisfying:
\begin{align*}
    &\cP_{\perp}(\bw(t+1))=\cP_{\perp}(\bw(t))-\eta\cP_\perp\left(\frac{\nabla\cL(\bw(t))}{\cL(\bw(t))}\right)
    \\=&\cP_{\perp}(\bw(t))-\eta\cP_\perp\left(\frac{\frac{1}{n}\sum_{i=1}^n e^{-\<\bw(t),\bz_i\>}\bz_i}{\frac{1}{n}\sum_{i=1}^n e^{-\<\bw(t),\bz_i\>}}\right)
    \\=&\cP_{\perp}(\bw(t))-\eta\frac{\sum_{i=1}^n e^{-\<\bw(t),\bz_i\>}\cP_\perp(\bz_i)}{\sum_{i=1}^n e^{-\<\bw(t),\bz_i\>}}.
\end{align*}

With the help of Theorem~\ref{thm: orthogonal dynamics, ji}, we know that 
\begin{align*}
    &\text{(L1) the minimizer $\bv^*\in{\rm span}\{\cP_{\perp}(\bx_i):i\in\cI\}$ (of $\cL_\perp(\cdot)$) is unique.}
    \\
    &\text{(L2) there exists an absolute constant $C>0$ such that $\norm{\cP_{\perp}(\bw(t))-\bv^\star}\leq C,\ \forall t$;}
    \\
    &\text{(L3) there exists $\mu>0$ such that $\cL_\perp(\cdot)$ is $\mu$-strongly convex in $\{\bv:\norm{\bv}<C+\norm{\bv^\star}\}$;}
\end{align*}
It is also easy to verify the $L$-smoothness:
\begin{align*}
    \text{(L4) there exists $L>0$ such that $\cL_\perp(\cdot)$ is $L$-smooth in $\{\bv:\norm{\bv}<C+\norm{\bv^\star}\}$.}
\end{align*}

\underline{Step I. The minimizer $\bv^\star\ne0$.}

If $\bv^\star=\bzero$, then $\nabla\cL_{\perp}(\bzero)=\bzero$, which implies
\begin{align*}
    \bzero=\frac{1}{|\cI|}\sum_{i\in\cI}e^0\cP(\bz_i)=\frac{1}{|\cI|}\sum_{i\in\cI}\cP(\bz_i).
\end{align*}
Therefore, 
\begin{align*}
    \frac{1}{|\cI|}\sum_{i\in\cI}\bz_i=&\frac{1}{|\cI|}\sum_{i\in\cI}\<\bz_i,\bw^\star\>\bw^\star+\frac{1}{|\cI|}\sum_{i\in\cI}\cP(\bz_i)
    \\=&\frac{1}{|\cI|}\sum_{i\in\cI}\gamma^\star\bw^\star=\gamma^\star\bw^\star,
\end{align*}
which is contradict to $\gamma^\star\bw^\star\ne\frac{1}{|\cI|}\sum_{i\in\cI}\bz_i$.

\underline{Step II. The gradient error near $\bzero\in{\rm span}\{\cP_{\perp}(\bx_i):i\in\cI\}$.}

Notice that the update rule of $\cP_\perp(\bw(t))$ can be written as 
\begin{align*}
    \cP_{\perp}(\bw(t+1))=
    \cP_{\perp}(\bw(t))-\eta\cP_{\perp}\left(\frac{\nabla\cL(\bw)}{\cL(\bw)}\right)
    =\cP_{\perp}(\bw(t))-\eta\Big(\nabla\cL_\perp(\cP_\perp(\bw(t)))+\beps(\bw(t))\Big).
\end{align*}

In this step, we will prove: there exists $\epsilon_0>0$ and $R_0>0$ such that the gradient error
\begin{align*}
    \norm{\beps(\bw)}=\norm{\cP_{\perp}\left(\frac{\nabla\cL(\bw)}{\cL(\bw)}\right)-\nabla\cL_{\perp}(\cP_\perp(\bw))}\leq\frac{1}{2}\norm{\nabla\cL_{\perp}(\cP_\perp(\bw))}
\end{align*}
holds for any $\bw$ satisfying $\<\bw,\bw^\star\>>R_0$ and $\norm{\cP_{\perp}(\bw)}<\epsilon_0$.

\underline{Step II (i).} 
 For some $\epsilon_1>0$, $\norm{\nabla\cL_{\perp}(\bv)-\nabla\cL_{\perp}(\bzero)}<\frac{1}{8}\norm{\nabla\cL_{\perp}(\bzero)}$ for any $\bv\in\bbB(\bzero;\epsilon_1)$.
    
    Notice that Step I ensures $\nabla\cL_{\perp}(\bzero)=\frac{1}{|\cI|}\sum_{i\in\cI}\cP(\bz_i)\neq\bzero$. We choose 
    \begin{align*}
        \epsilon_1=\min\left\{\frac{\norm{\nabla\cL_{\perp}(\bzero)}}{8L},C+\norm{\bv^\star}\right\}.
    \end{align*}
    Then for any $\bv\in\bbB(\bzero,\epsilon_1)$, then (iv) ($L$-smooth) ensures that:
    \begin{align*}
        \norm{\nabla\cL_{\perp}(\bv)-\nabla\cL_{\perp}(\bzero)}
        \leq L\norm{\bv-\bzero}<{L\epsilon_1}\leq\frac{1}{8}\norm{\nabla\cL_{\perp}(\bzero)}.
    \end{align*}

\underline{Step II (ii).} For some $\epsilon_2>0$ and $R_0>0$, $\norm{\cP_{\perp}\left(\frac{\nabla\cL(\bw)}{\cL(\bw)}\right)-\nabla\cL_{\perp}(\bzero)}\leq\frac{1}{8}\norm{\nabla\cL_{\perp}(\bzero)}$ holds for any $\bw$ satisfying $\<\bw,\bw^\star\>>R_0$ and $\norm{\cP_{\perp}(\bw)}<\epsilon_2$.

    Due to $\norm{\nabla\cL_{\perp}(\bzero)}/16\ne0$, using Lemma~\ref{lemma: perp gradient error estimate}, there exists $\epsilon_2>0$ and $R_2>0$ such that: for any $\bw$ satisfying $\<\bw,\bw^\star\>>R_0$ and $\norm{\cP_{\perp}(\bw)}<\epsilon_2$, it holds that:
    \begin{align*}
    &\norm{\cP_\perp\left(\frac{\nabla\cL(\bw)}{\cL(\bw)}\right)-\nabla\cL_\perp(\bzero)}
    \\=&
    \norm{\frac{\sum_{i=1}^n e^{-\<\bw,\bz_i\>}\cP_\perp(\bz_i)}{\sum_{j=1}^n e^{-\<\bw,\bz_j\>}}-\frac{1}{|\cI|}\sum_{i\in\cI}\cP_\perp(\bz_i)}
    \\\leq&
    \norm{\frac{\sum_{i=1}^n e^{-\<\bw,\bz_i\>}\cP_\perp(\bz_i)}{\sum_{j=1}^n e^{-\<\bw,\bz_j\>}}-\frac{\sum_{i\in\cI} e^{-\<\bw,\bz_i\>}\cP_\perp(\bz_i)}{\sum_{j\in\cI} e^{-\<\bw,\bz_j\>}}}
    \\&\quad+\norm{\frac{\sum_{i\in\cI} e^{-\<\bw,\bz_i\>}\cP_\perp(\bz_i)}{\sum_{j\in\cI} e^{-\<\bw,\bz_j\>}}-\frac{1}{|\cI|}\sum_{i\in\cI}\cP_\perp(\bz_i)}
    \\
    \overset{\text{Lemma~\ref{lemma: perp gradient error estimate} (iii)(iv)}}{\leq}&
    \frac{\norm{\nabla\cL_{\perp}(\bzero)}}{16}+\frac{\norm{\nabla\cL_{\perp}(\bzero)}}{16}=\frac{\norm{\nabla\cL_{\perp}(\bzero)}}{8}.
\end{align*}

\underline{Step II (iii).} Based on Step II (i) and (ii), we can select
\begin{align*}
    \epsilon_3:=\min\{\epsilon_1,\epsilon_2\},\quad R_0:=R_0.
\end{align*}
Then for  any $\bw$ satisfying $\<\bw,\bw^\star\>>R_0$ and $\norm{\cP_{\perp}(\bw)}<\epsilon_3$, it holds that:
\begin{align*}
    &\norm{\beps(\bw)}=\norm{\cP_{\perp}\left(\frac{\nabla\cL(\bw)}{\cL(\bw)}\right)-\nabla\cL_{\perp}(\cP_\perp(\bw))}
    \\\leq&
    \norm{\cP_{\perp}\left(\frac{\nabla\cL(\bw)}{\cL(\bw)}\right)-\nabla\cL_{\perp}(\bzero)}+\norm{\nabla\cL_{\perp}(\bzero)-\nabla\cL_{\perp}(\cP_\perp(\bw))}
    \\
    \overset{\text{Step II (i) and (ii)}}{\leq}&
    \frac{\norm{\nabla\cL_{\perp}(\bzero)}}{8}+\frac{\norm{\nabla\cL_{\perp}(\bzero)}}{8}=\frac{\norm{\nabla\cL_{\perp}(\bzero)}}{4}\overset{\text{Step II (i)}}{<}\frac{\norm{\nabla\cL_{\perp}(\cP_\perp(\bw))}}{2}.
\end{align*}

\underline{Step III. The proof of the main claim.}

From Theorem~\ref{thm: NGD upper bound ji}, we know $\norm{\bw(t)}=\Theta(t)$ and $\norm{\frac{\bw(t)}{\norm{\bw(t)}}-\bw^\star}=\cO(1/t)$. Therefore, there exists $T_0>0$ such that $\<\bw(t),\bw^\star\>>R_0$ holds for any $t>T_0$ (where $R_0$ is defined in Step II).
Additionally, we choose
\begin{align*}
    \epsilon_0:=\min\left\{\epsilon_3,\frac{1}{2}\norm{\bv^\star}\right\},
\end{align*}
where $\epsilon_3$ is defined in Step II.

Consequently, in this step, we aim to prove: there exists constant $T_0>0$ and $\epsilon_0>0$, such that: for any $T>T_0$, there exists $t>T$ s.t. $\norm{\cP_{\perp}(\bw(t))}>\epsilon_0$.

Given any $T>T_0$, now we assume that $\norm{\cP_{\perp}(\bw(t))}<\epsilon_0$ holds for any $t>T$.


Recalling Theorem~\ref{thm: orthogonal dynamics, ji}, it ensures that $\cL_\perp(\cdot)$ is $\mu$-strongly convex in $\bbB(\bv^\star;C)-\bbB(\bv^\star;\delta)$ for some $\mu>0$. Therefore, 
\begin{align*}
    \norm{\nabla\cL_\perp(\vw)}\geq\mu\norm{\bv-\bv^\star}\geq\mu\delta,\quad\forall\bv\in\bbB(\bv^\star;C)-\bbB(\bv^\star;\delta).
\end{align*}

By our result in Step II, for any $t>T$, the gradient error holds that 
\begin{align*}
    \norm{\cP_\perp\left(\frac{\nabla\cL(\bw(t))}{\cL(\bw(t))}\right)-\nabla\cL_\perp(\cP_{\perp}(\bw(t)))}\leq\frac{1}{2}\norm{\nabla\cL_\perp(\cP_{\perp}(\bw(t)))}.
\end{align*}


Hence, by setting $\eta\leq 1/9L$, the loss descent has the following lower bound: for any $t> T$,
\begin{align*}
    &\cL_{\perp}(\cP_{\perp}(\bw(t)))-\cL_{\perp}^\star=\cL_{\perp}\left(\cP_{\perp}(\bw(t-1))-\eta\cP_\perp\left(\frac{\nabla\cL(\bw(t-1))}{\cL(\bw(t-1))}\right)\right)-\cL_{\perp}^\star
    \\
    \overset{\text{Lemma~\ref{lemma: convex optimization}}}{\leq}&
    \cL_{\perp}\left(\cP_{\perp}(\bw(t-1))\right)-\cL_{\perp}^\star-\eta\<\nabla\cL_{\perp}(\cP_{\perp}(\bw(t-1))),\cP_\perp\left(\frac{\nabla\cL(\bw(t-1))}{\cL(\bw(t-1))}\right)\>
    \\&\quad+\frac{L}{2}\eta^2\norm{\cP_\perp\left(\frac{\nabla\cL(\bw(t-1))}{\cL(\bw(t-1))}\right)}^2
    \\=&
    \cL_{\perp}\left(\cP_{\perp}(\bw(t-1))\right)-\cL_{\perp}^\star-\eta\norm{\nabla\cL_{\perp}(\cP_{\perp}(\bw(t-1)))}^2
    \\&\quad-\eta\<\nabla\cL_{\perp}(\cP_{\perp}(\bw(t-1))),\cP_\perp\left(\frac{\nabla\cL(\bw(t-1))}{\cL(\bw(t-1))}\right)-\nabla\cL_{\perp}(\cP_{\perp}(\bw(t-1)))\>
    \\&\quad+\frac{L}{2}\eta^2\norm{\cP_\perp\left(\frac{\nabla\cL(\bw(t-1))}{\cL(\bw(t-1))}\right)}^2
    \\\leq&
    \cL_{\perp}\left(\cP_{\perp}(\bw(t-1))\right)-\cL_{\perp}^\star
    \\&\quad-\eta\left(\norm{\nabla\cL_{\perp}(\cP_{\perp}(\bw(t-1)))}^2-\frac{1}{2}\norm{\nabla\cL_{\perp}(\cP_{\perp}(\bw(t-1)))}^2\right)
    \\&\quad+\frac{L\eta^2}{2}\left(\frac{3}{2}\right)^2\norm{\nabla\cL_{\perp}(\cP_{\perp}(\bw(t-1)))}^2
    \\\leq&
    \cL_{\perp}\left(\cP_{\perp}(\bw(t-1))\right)-\cL_{\perp}^\star-\eta\left(\frac{1}{2}-\frac{9\eta L}{8}\right)\norm{\nabla\cL_{\perp}(\cP_{\perp}(\bw(t-1)))}^2
    \\\leq&
    \cL_{\perp}\left(\cP_{\perp}(\bw(t-1))\right)-\cL_{\perp}^\star-\frac{3\eta}{8}\norm{\nabla\cL_{\perp}(\cP_{\perp}(\bw(t-1)))}^2
    \\
    \overset{\text{Lemma~\ref{lemma: convex optimization}}}{\leq}&
    \cL_{\perp}\left(\cP_{\perp}(\bw(t-1))\right)-\cL_{\perp}^\star-\frac{3\eta}{8}\cdot 2\mu\left(\cL_{\perp}\left(\cP_{\perp}(\bw(t-1))\right)-\cL_{\perp}^\star\right)
    \\\leq&\left(1-\frac{3\eta\mu}{4}\right)\left(\cL_{\perp}\left(\cP_{\perp}(\bw(t-1))\right)-\cL_{\perp}^\star\right)
    \\\leq&\cdots
    \\\leq&\left(1-\frac{3\eta\mu}{4}\right)^{t-T}\left(\cL_{\perp}\left(\cP_{\perp}(\bw(T_\epsilon))\right)-\cL_{\perp}^\star\right).
\end{align*}

Hence, there exists time $t>T$ such that $\cL_\perp\left(\cP_{\perp}(\bw(t))\right)-\cL_\perp^\star<\frac{\mu\epsilon_0}{4}$. 

On the other hand, the strong convexity and Lemma~\ref{lemma: convex optimization} implies that
\begin{align*}
    &\cL_\perp\left(\cP_{\perp}(\bw(t))\right)-\cL_\perp^\star\geq\frac{\mu}{2}\norm{\cP_{\perp}(\bw(t))-\bv^\star}\geq\frac{\mu}{2}\left(\norm{\bv^\star}-\norm{\cP_{\perp}(\bw(t)}\right)
    \\>&\frac{\mu}{2}(\norm{\bv^\star}-\epsilon_0)\geq\frac{\mu\epsilon_0}{4}.
\end{align*}
Thus, we obtain the {\bf contradiction}. Hence, our {\bf main claim} holds.

\underline{Step IV. Final Lower bound.}

From our result in Step III,  there exists a subsequence $\bw(t_k)$ satisfying $t_{k+1}>t_k$, $t_k\to\infty$, and $\norm{\cP_\perp(\bw(t_k))}>\epsilon_0$.
Recalling (L2), it holds that $\norm{\cP_{\perp}(\bw(t_k))}\leq C+\norm{\bv^\star}$.
Therefore, 
$\norm{\cP_{\perp}\bw(t_k)}=\Theta(1)$.

Recalling Theorem~\ref{thm: NGD upper bound ji}, it holds that $\norm{\bw(t_k)}=\Theta(t_k)$.
Then, a direct calculation ensures that:
\begin{align*}
    &\norm{\frac{\bw(t_k)}{\norm{\bw(t_k)}}-\bw^\star}^2=2-2\frac{\<\bw(t_k),\bw^\star\>}{\norm{\bw(t_k)}}
    \\=&2-2\frac{\<\bw(t_k),\bw^\star\>}{\sqrt{\<\bw(t_k),\bw^\star\>^2+\norm{\cP_\perp(\bw(t_k))}^2}}
    \\=&2-\frac{2}{\sqrt{1+\frac{\norm{\cP_\perp(\bw(t_k))}^2}{\<\bw(t_k),\bw^\star\>^2}}}=\Theta\left(\frac{\norm{\cP_\perp(\bw(t_k))}^2}{\<\bw(t_k),\bw^\star\>^2}\right)
    \\=&\Theta\left(\frac{\norm{\cP_\perp(\bw(t_k))}^2}{\norm{\bw(t_k)}^2-\norm{\cP_\perp(\bw(t_k))}^2}\right)=\Theta\left(\frac{1}{\norm{\bw(t_k)}^2-\norm{\cP_\perp(\bw(t_k))}^2}\right)=\Theta\left(\frac{1}{t_k^2}\right),
\end{align*}
which implies the tight bound for the directional convergence rate: $\norm{\frac{\bw(t_k)}{\norm{\bw(t_k)}}-\bw^\star}=\Theta\left(\frac{1}{t_k}\right)$.
Moreover, with the help of Lemma~\ref{lemma: tight margin error and Directional error}, we have the lower bound for the margin maximization rate: $\gamma^\star-\gamma(\bw(t_k))=\Omega\left(\frac{1}{t_k^2}\right)$.

Hence, we have proved Theorem~\ref{thm: GD NGD main result} for NGD.


\end{proof}

\subsection{Useful Lemmas}

\begin{lemma}\label{lemma: uniform project lower bound}
    Under Assumption~\ref{ass: linearly separable}, it holds that
    \begin{align*}
        \gamma^\star\leq\left<-\frac{\nabla\cL(\bw)}{\cL(\bw)},\bw^\star\right>\leq1,\quad
        \gamma^\star\leq\norm{\frac{\nabla\cL(\bw)}{\cL(\bw)}}\leq1,\quad\forall \bw\in\bbR^{d}.
    \end{align*}
\end{lemma}

\begin{proof}[Proof of Lemma~\ref{lemma: uniform project lower bound}] For any $\bw\in\bbR^{d}$, we have:
\begin{align*}
    &\left<-\frac{\nabla\cL(\bw)}{\cL(\bw)},\bw^\star\right>=\frac{\frac{1}{n}\sum\limits_{i=1}^n e^{-y_i\<\bw,\bx_i\>}y_i\<\bw^\star,\bx_i\>}{\frac{1}{n}\sum\limits_{i=1}^ne^{-y_i\<\bw,\bx_i\>}}
    \geq\frac{\frac{1}{n}\sum\limits_{i=1}^n e^{-y_i\<\bw,\bx_i\>}\gamma^{\star}}{\frac{1}{n}\sum\limits_{i=1}^ne^{-y_i\<\bw,\bx_i\>}}=\gamma^*,
    \\
    &\left<-\frac{\nabla\cL(\bw)}{\cL(\bw)},\bw^\star\right>=\frac{\frac{1}{n}\sum\limits_{i=1}^n e^{-y_i\<\bw,\bx_i\>}y_i\<\bw^\star,\bx_i\>}{\frac{1}{n}\sum\limits_{i=1}^ne^{-y_i\<\bw,\bx_i\>}}
    \leq\frac{\frac{1}{n}\sum\limits_{i=1}^n e^{-y_i\<\bw,\bx_i\>}}{\frac{1}{n}\sum\limits_{i=1}^ne^{-y_i\<\bw,\bx_i\>}}=1.
\end{align*}

For the lower bound of $\norm{\nabla\cL(\bw)/\cL(\bw)}$, it holds that
\begin{align*}
    \norm{\frac{\nabla\cL(\bw)}{\cL(\bw)}}\geq\<-\frac{\nabla\cL(\bw)}{\cL(\bw)},\bw^\star\>\geq \gamma^\star.
\end{align*}

For the upper bound of $\norm{\nabla\cL(\bw)/\cL(\bw)}$, it holds that
\begin{align*}
    &\norm{\frac{\nabla\cL(\bw)}{\cL(\bw)}}=\norm{-\frac{\frac{1}{n}\sum\limits_{i=1}^ne^{-y_i\<\bw,\bx_i\>}y_i\bx_i}{\frac{1}{n}\sum\limits_{i=1}^ne^{-y_i\<\bw,\bx_i\>}}}\leq\frac{\frac{1}{n}\sum\limits_{i=1}^ne^{-y_i\<\bw,\bx_i\>}\norm{y_i\bx_i}}{\frac{1}{n}\sum\limits_{i=1}^ne^{-y_i\<\bw,\bx_i\>}}\leq1.
\end{align*}

\end{proof}

\begin{lemma}[(Two-sided) Margin error and Directional error]\label{lemma: tight margin error and Directional error}
    Under Assumption~\ref{ass: linearly separable} and~\ref{ass: non-degenerate data}, if $\bw$ satisfies $\norm{\bw^\star-\frac{\bw}{\norm{\bw}}}<(\gamma_{\rm sub}^\star-\gamma^\star)/2$ (where $\gamma_{\rm sub}^\star=\min\limits_{i\notin\cI}\<\bw^\star,\bz_i\>$), then it holds that 
    \begin{align*}
        \frac{\gamma^\star}{2}\norm{\bw^\star-\frac{\bw}{\norm{\bw}}}^2\leq\gamma^\star-\gamma(\bw)\leq\norm{\bw^\star-\frac{\bw}{\norm{\bw}}}.
    \end{align*}
\end{lemma}

\begin{proof}[Proof of Lemma~\ref{lemma: tight margin error and Directional error}]\ \\
This lemma is an improved version of Lemma~\ref{lemma: Margin error and Directional error}. The second ``$\leq$'' is ensured by Lemma~\ref{lemma: Margin error and Directional error}, and we only need to prove the first ``$\leq$''. For simplicity, we still denote $\bz_i:=\bz_i,i\in[n]$.

\underline{Step I. $\gamma(\bw)=\min_{i\in\cI}\<\frac{\bw}{\norm{\bw}},\bz_i\>$.}

For any $i\in[n]$, we have
\begin{align*}
    \left|\<\bw^\star,\bz_i\>-\<\frac{\bw}{\norm{\bw}},\bz_i\>\right|\leq\norm{\bw^\star-\frac{\bw}{\norm{\bw}}}<\frac{\gamma^\star_{\rm sub}-\gamma^\star}{2},
\end{align*}
which implies $\<\bw^\star,\bz_i\>-\frac{\gamma^\star_{\rm sub}-\gamma^\star}{2}<\<\frac{\bw}{\norm{\bw}},\bz_i\><\<\bw^\star,\bz_i\>+\frac{\gamma^\star_{\rm sub}-\gamma^\star}{2}$. Furthermore,
\begin{align*}
    &\<\frac{\bw}{\norm{\bw}},\bz_i\><\gamma^\star+\frac{\gamma^\star_{\rm sub}-\gamma^\star}{2}=\frac{\gamma^\star_{\rm sub}+\gamma^\star}{2},\quad i\in\cI;
    \\&
    \<\frac{\bw}{\norm{\bw}},\bz_i\>>\gamma_{\rm sub}^\star-\frac{\gamma^\star_{\rm sub}-\gamma^\star}{2}=\frac{\gamma^\star_{\rm sub}+\gamma^\star}{2},\quad i\notin\cI.
\end{align*}
Therefore, it holds that
\begin{align*}
    \gamma(\bw)=\min_{i\in[n]}\<\frac{\bw}{\norm{\bw}},\bz_i\>=\min_{i\in\cI}\<\frac{\bw}{\norm{\bw}},\bz_i\>.
\end{align*}

\underline{Step II. The lower bound for $\gamma^\star-\gamma(\bw)$.}

From the result in Step I,
\begin{align*}
    &\gamma^\star-\gamma(\bw)=\gamma^\star-\min_{i\in\cI}\<\frac{\bw}{\norm{\bw}},\bz_i\>
    =\max_{i\in\cI}\left(\gamma^\star-\<\frac{\bw}{\norm{\bw}},\bz_i\>\right)
    \\=&\max_{i\in\cI}\left(\<\bw^\star,\bz_i\>-\<\frac{\bw}{\norm{\bw}},\bz_i\>\right)
    =\max_{i\in\cI}\<\bw^\star-\frac{\bw}{\norm{\bw}},\bz_i\>.
\end{align*}
Recalling Assumption~\ref{ass: non-degenerate data}, $\bw^\star=\sum_{i\in\cI}\alpha_i\bz_i$, where $\alpha_i>0$ and $1/\gamma^\star=\sum_{i\in\cI}\alpha_i$. Thus, for every $\alpha_i>0$, we have $\alpha_i(\gamma^\star-\gamma(\bw))=\alpha_i\max_{i\in\cI}\<\bw^\star-\frac{\bw}{\norm{\bw}},\bz_i\>$, which ensures:
\begin{align*}
    &\left(\sum_{i\in\cI}\alpha_i\right)(\gamma^\star-\gamma(\bw))=\left(\sum_{i\in\cI}\alpha_i\right)\max_{i\in\cI}\<\bw^\star-\frac{\bw}{\norm{\bw}},\bz_i\>
    \\\geq&
    \sum_{i\in\cI}\alpha_i\<\bw^\star-\frac{\bw}{\norm{\bw}},\bz_i\>=\<\bw^\star-\frac{\bw}{\norm{\bw}},\sum_{i\in\cI}\alpha_i\bz_i\>
    \\=&
    \<\bw^\star-\frac{\bw}{\norm{\bw}},\bw^\star\>=1-\<\frac{\bw}{\norm{\bw}},\bw^\star\>=\frac{1}{2}\left(2-2\<\frac{\bw}{\norm{\bw}},\bw^\star\>\right)
    \\=&\frac{1}{2}\norm{\bw^\star-\frac{\bw}{\norm{\bw}}}^2.
\end{align*}
Hence, we obtain:
\begin{align*}
    \gamma^\star-\gamma(\bw)\geq\frac{1}{2\sum_{i\in\cI}\alpha_i}\norm{\bw^\star-\frac{\bw}{\norm{\bw}}}^2=\frac{\gamma^\star}{2}\norm{\bw^\star-\frac{\bw}{\norm{\bw}}}^2.
\end{align*}

\end{proof}

\begin{lemma}
\label{lemma: perp gradient error estimate}
For any $\epsilon>0$, there exists $\delta\in(0,1)$ and $R>0$, such that for any $\bw$ satisfying $\norm{\cP_\perp(\bw)}<\delta$ and $\<\bw,\bw^\star\>>R$, it holds that
    \begin{align*}
    &(i).\ \left|\sum_{i=1}^n e^{-\<\bw,\bz_i\>}-\sum_{i\in\cI} e^{-\<\bw,\bz_i\>}\right|<\epsilon\sum_{i\in\cI} e^{-\<\bw,\bz_i\>};
    \\
    &(ii).\ \norm{\sum_{i=1}^n e^{-\<\bw,\bz_i\>}\cP_\perp(\bz_i)-\sum_{i\in\cI} e^{-\<\bw,\bz_i\>}\cP_\perp(\bz_i)}<\epsilon\sum_{i\in\cI} e^{-\<\bw,\bz_i\>};
    \\
    &(iii).\norm{\frac{\sum_{i=1}^n e^{-\<\bw,\bz_i\>}\cP_\perp(\bz_i)}{\sum_{j=1}^n e^{-\<\bw,\bz_j\>}}-\frac{\sum_{i\in\cI} e^{-\<\bw,\bz_i\>}\cP_\perp(\bz_i)}{\sum_{j\in\cI} e^{-\<\bw,\bz_j\>}}}<\epsilon;
    \\
    &(iv).\ \norm{\frac{\sum_{i\in\cI} e^{-\<\bw,\bz_i\>}\cP_\perp(\bz_i)}{\sum_{j\in\cI} e^{-\<\bw,\bz_j\>}}-\frac{1}{|\cI|}\sum_{i\in\cI}\cP_\perp(\bz_i)}<\epsilon.
\end{align*}
\end{lemma}

\begin{proof}[Proof of Lemma~\ref{lemma: perp gradient error estimate}]
\ \\
For simplicity, in this proof, we still denote $\bz_i:=\bx_i y_i$ ($i\in[n]$) and $\gamma^\star_{\rm sub}:=\min\limits_{i\notin\cI}\<\frac{\bw}{\norm{\bw}},\bz_i\>$. And we are given an $\epsilon>0$. Without loss of generality, we can assume $\delta<1.$



\underline{Proof of (i).}

Notice the following two estimates:
\begin{align*}
    &\sum_{i\notin\cI}e^{-\<\bw,\bz_i\>}=\sum_{i\notin\cI}e^{-\<\bw,\bw^\star\>\<\bw^\star,\bz_i\>}e^{-\<\cP_\perp(\bw),\bz_i\>}
    \\\leq&
    \sum_{i\notin\cI}e^{-\<\bw,\bw^\star\>\gamma_{\rm sub}^\star}e^{-\<\cP_\perp(\bw),\bz_i\>}
    \leq
    \sum_{i\notin\cI}e^{-\<\bw,\bw^\star\>\gamma_{\rm sub}^\star}e^{\norm{\cP_\perp(\bw)}},
\end{align*}
\begin{align*}
    &\sum_{i\in\cI}e^{-\<\bw,\bz_i\>}=\sum_{i\in\cI}e^{-\<\bw,\bw^\star\>\<\bw^\star,\bz_i\>}e^{-\<\cP_\perp(\bw),\bz_i\>}
    \\=&
    \sum_{i\in\cI}e^{-\<\bw,\bw^\star\>\gamma^\star}e^{-\<\cP_\perp(\bw),\bz_i\>}
    \geq\sum_{i\in\cI}e^{-\<\bw,\bw^\star\>\gamma^\star}e^{-\norm{\cP_\perp(\bw)}}.
\end{align*}
Then for any $\delta>0$, $R_>0$, and $\bw$ satisfying $\norm{\cP_\perp(\bw)}<\delta$ and $\<\bw,\bw^\star\>>R$, we have:
\begin{align*}
    &\frac{\sum_{i\notin\cI}e^{-\<\bw,\bz_i\>}}{\sum_{i\in\cI}e^{-\<\bw,\bz_i\>}}\leq\frac{\sum_{i\notin\cI}e^{-\<\bw,\bw^\star\>\gamma_{\rm sub}^\star}e^{\norm{\cP_\perp(\bw)}}}{\sum_{i\in\cI}e^{-\<\bw,\bw^\star\>\gamma^\star}e^{-\norm{\cP_\perp(\bw)}}}
    \\\leq&
    \frac{n-|\cI|}{|\cI|}\exp\left(-\<\bw,\bw^\star\>(\gamma_{\rm sub}^\star-\gamma^\star)+2\norm{\cP_\perp(\bw)}\right)
    \\\leq&
    \frac{n-|\cI|}{|\cI|}\exp\left(-\left(1-\frac{\delta^2}{2}\right)\norm{\bw}(\gamma_{\rm sub}^\star-\gamma^\star)+\sqrt{2}\delta\norm{\bw}\right)
    \\=&
    \frac{n-|\cI|}{|\cI|}\exp\left(-\norm{\bw}\left(\left(1-\frac{\delta^2}{2}\right)(\gamma^\star_{\rm sub}-\gamma^\star)+\sqrt{2}\delta\right)\right)
    \\\leq&
    \frac{n-|\cI|}{|\cI|}\exp\left(-\<\bw,\bw^\star\>\left(\left(1-\frac{\delta^2}{2}\right)(\gamma^\star_{\rm sub}-\gamma^\star)+\sqrt{2}\delta\right)\right).
\end{align*}
Due to $\gamma^\star_{\rm sub}-\gamma^\star>0$, there exist constants $\delta_1>0$ and $R_1>0$ such that: for any $\bw$ satisfying $\norm{\cP_\perp(\bw)}<\delta_1$ and $\<\bw,\bw^\star\>>R_1$, it holds $\frac{\sum_{i\notin\cI}e^{-\<\bw,\bz_i\>}}{\sum_{i\in\cI}e^{-\<\bw,\bz_i\>}}<\epsilon$, which means (i) holds.

\underline{Proof of (ii).}

Based on the proof of (i), there exists $\delta_1>0$ and $R_1>0$ such that: for any $\bw$ satisfying $\norm{\cP_\perp(\bw)}<\delta_1$ and $\<\bw,\bw^\star\>>R_1$, it holds that $\frac{\sum_{i\notin\cI}e^{-\<\bw,\bz_i\>}}{\sum_{i\in\cI}e^{-\<\bw,\bz_i\>}}<\epsilon$, which means
\begin{align*}
    &\norm{\sum_{i\in[n]}e^{-\<\bw,\bz_i\>}\cP_\perp(z_i)-\sum_{i\in\cI}e^{-\<\bw,\bz_i\>}\cP_\perp(z_i)}=\norm{\sum_{i\notin\cI}e^{-\<\bw,\bz_i\>}\cP_\perp(z_i)}
    \\\leq&
    \sum_{i\notin\cI}e^{-\<\bw,\bz_i\>}<\epsilon\sum_{i\in\cI}e^{-\<\bw,\bz_i\>}.
\end{align*}

\underline{Proof of (iii).} 

From the results of (i)(ii), for $\epsilon/2$, there exists $\delta_2>0$ and $R_2>0$ such that: for any $\bw$ satisfying $\norm{\cP_\perp(\bw)}<\delta_2$ and $\<\bw,\bw^\star\>>R_2$,
\begin{align*}
    &\left|\sum_{i=1}^n e^{-\<\bw,\bz_i\>}-\sum_{i\in\cI} e^{-\<\bw,\bz_i\>}\right|<\frac{\epsilon}{2}\sum_{i\in\cI} e^{-\<\bw,\bz_i\>};
    \\
    &\norm{\sum_{i=1}^n e^{-\<\bw,\bz_i\>}\cP_\perp(\bz_i)-\sum_{i\in\cI} e^{-\<\bw,\bz_i\>}\cP_\perp(\bz_i)}<\frac{\epsilon}{2}\sum_{i\in\cI} e^{-\<\bw,\bz_i\>};
\end{align*}
Therefore, we have 
\begin{align*}
    &\norm{\frac{\sum_{i\in[n]} e^{-\<\bw,\bz_i\>}\cP_\perp(\bz_i)}{\sum_{j\in[n]} e^{-\<\bw,\bz_j\>}}-\frac{\sum_{i\in\cI} e^{-\<\bw,\bz_i\>}\cP_\perp(\bz_i)}{\sum_{j\in\cI} e^{-\<\bw,\bz_j\>}}}
    \\\leq&
    \norm{\frac{\sum_{i\in[n]} e^{-\<\bw,\bz_i\>}\cP_\perp(\bz_i)}{\sum_{j\in[n]} e^{-\<\bw,\bz_j\>}}-\frac{\sum_{i\in\cI} e^{-\<\bw,\bz_i\>}\cP_\perp(\bz_i)}{\sum_{j\in[n]} e^{-\<\bw,\bz_j\>}}}
    \\&\quad+\norm{\frac{\sum_{i\in\cI} e^{-\<\bw,\bz_i\>}\cP_\perp(\bz_i)}{\sum_{j\in[n]} e^{-\<\bw,\bz_j\>}}-\frac{\sum_{i\in\cI} e^{-\<\bw,\bz_i\>}\cP_\perp(\bz_i)}{\sum_{j\in\cI} e^{-\<\bw,\bz_j\>}}}
    \\\leq&
    \frac{\epsilon}{2}\frac{\sum_{i\in\cI}e^{-\<\bw,\bz_i\>}}{\sum_{j\in[n]} e^{-\<\bw,\bz_j\>}}+\norm{\sum_{i\in\cI} e^{-\<\bw,\bz_i\>}\cP_\perp(\bz_i)}\left|\frac{\sum_{j\in\cI} e^{-\<\bw,\bz_j\>}-\sum_{j\in[n]}e^{-\<\bw,\bz_j\>}}{\left(\sum_{j\in\cI} e^{-\<\bw,\bz_j\>}\right)\left(\sum_{j\in[n]} e^{-\<\bw,\bz_j\>}\right)}\right|
    \\\leq&
    \frac{\epsilon}{2}\frac{\sum_{i\in\cI}e^{-\<\bw,\bz_i\>}}{\sum_{j\in[n]} e^{-\<\bw,\bz_j\>}}+\frac{\epsilon}{2}\frac{\norm{\sum_{i\in\cI} e^{-\<\bw,\bz_i\>}\cP_\perp(\bz_i)}}{\sum_{j\in[n]} e^{-\<\bw,\bz_j\>}}
    \\\leq&
    \frac{\epsilon}{2}+\frac{\epsilon}{2}=\epsilon/2.
\end{align*}

\underline{Proof of (iv).}

There exists $\delta_3>0$ such that: for any $\bw$ satisfying $\norm{\cP_\perp(\bw)}<\delta_3$, 
\begin{align*}
    \left|e^{-\<\cP_\perp(\bw),\bz_i\>}-1\right|
    \leq 2\left|\<\cP_\perp(\bw),\bz_i\>\right|<2\norm{\cP_\perp(\bw)}<\epsilon/4.
\end{align*}
Then we have 
\begin{align*}
    \left|\sum_{j\in\cI}e^{-\<\cP_\perp(\bw),\bz_i\>}-|\cI|\right|
    \leq
    \sum_{j\in\cI}\left|e^{-\<\cP_\perp(\bw),\bz_i\>}-1\right|\leq\epsilon|\cI|/4.
\end{align*}
Thus, we can derive:
\begin{align*}
    &\norm{\frac{\sum_{i\in\cI} e^{-\<\bw,\bz_i\>}\cP_\perp(\bz_i)}{\sum_{j\in\cI} e^{-\<\bw,\bz_j\>}}-\frac{1}{|\cI|}\sum_{i\in\cI}\cP_\perp(\bz_i)}
    \\=&
    \norm{\frac{\sum_{i\in\cI} e^{-\<\bw,\bw^\star\>\<\bw^\star,\bz_i\>}e^{-\<\cP_\perp(\bw),\bz_i\>}\cP_\perp(\bz_i)}{\sum_{j\in\cI} e^{-\<\bw,\bw^\star\>\<\bw^\star,\bz_j\>}e^{-\<\cP_\perp(\bw),\bz_j\>}}-\frac{1}{|\cI|}\sum_{i\in\cI}\cP_\perp(\bz_i)}
    \\=&
    \norm{\frac{e^{-\<\bw,\bw^\star\>\gamma^\star}\sum_{i\in\cI}e^{-\<\cP_\perp(\bw),\bz_i\>}\cP_\perp(\bz_i)}{e^{-\<\bw,\bw^\star\>\gamma^\star}\sum_{j\in\cI}e^{-\<\cP_\perp(\bw),\bz_j\>}}-\frac{1}{|\cI|}\sum_{i\in\cI}\cP_\perp(\bz_i)}
    \\=&
    \norm{\frac{\sum_{i\in\cI}e^{-\<\cP_\perp(\bw),\bz_i\>}\cP_\perp(\bz_i)}{\sum_{j\in\cI}e^{-\<\cP_\perp(\bw),\bz_j\>}}-\frac{1}{|\cI|}\sum_{i\in\cI}\cP_\perp(\bz_i)}
    \\=&
    \norm{\sum_{i\in\cI}\left(\frac{e^{-\<\cP_\perp(\bw),\bz_i\>}}{\sum_{j\in\cI}e^{-\<\cP_\perp(\bw),\bz_j\>}}-\frac{1}{|\cI|}\right)\cP_\perp(\bz_i)}
    \\\leq&
    \sum_{j\in\cI}\left|\frac{e^{-\<\cP_\perp(\bw),\bz_i\>}}{\sum_{j\in\cI}e^{-\<\cP_\perp(\bw),\bz_j\>}}-\frac{1}{|\cI|}\right|
    \\\leq&
    \sum_{j\in\cI}\left|\frac{e^{-\<\cP_\perp(\bw),\bz_i\>}}{\sum_{j\in\cI}e^{-\<\cP_\perp(\bw),\bz_j\>}}-\frac{1}{\sum_{j\in\cI}e^{-\<\cP_\perp(\bw),\bz_j\>}}\right|+\sum_{j\in\cI}\left|\frac{1}{\sum_{j\in\cI}e^{-\<\cP_\perp(\bw),\bz_j\>}}-\frac{1}{|\cI|}\right|
    \\\leq&
    \sum_{j\in\cI}\frac{\epsilon}{4\sum_{j\in\cI}e^{-\<\cP_\perp(\bw),\bz_j\>}}+\sum_{j\in\cI}\frac{\epsilon|\cI|}{4\sum_{j\in\cI}e^{-\<\cP_\perp(\bw),\bz_j\>}|\cI|}
    \\\leq&
    \sum_{j\in\cI}\frac{\epsilon}{4(1-\epsilon)|\cI|}+\sum_{j\in\cI}\frac{\epsilon}{4(1-\epsilon)|\cI|}=\frac{\epsilon}{2(1-\epsilon)}<\epsilon.
\end{align*}

\underline{The final results.}

We choose $\delta=\min\{\delta_2,\delta_3\}$ and $R=R_2$. From our proofs above, (i)$\sim$(iv) all hold for any $\bw$ satisfying $\norm{\cP_\perp(\bw)}<\delta$ and $\<\bw,\bw^\star\>>R$.

\end{proof}

\begin{lemma}[\citep{ji2020gradient}]\label{lemma: GD directional convergence}
Under Assumption~\ref{ass: linearly separable}, let $\bw(t)$ be trained by GD~\eqref{equ: GD} with $\eta\leq1/2$ starting from $\bw(0)=\bzero$, then GD converges to the max-margin direction: 
\begin{align*}
    \lim_{t\to+\infty}\frac{\bw(t)}{\norm{\bw(t)}}\to\bw^\star.
\end{align*}
\end{lemma}

\begin{lemma}[Theorem 4.3,  \citep{ji2021characterizing}]\label{thm: NGD upper bound ji}
Under Assumption~\ref{ass: linearly separable} and~\ref{ass: non-degenerate data}, 

(I) (GD). let $\bw(t)$ be trained by GD~\eqref{equ: GD} with $\eta\leq1$ starting from $\bw(0)=\bzero$. Then $\norm{\frac{\bw(t)}{\norm{\bw(t)}}-\bw^\star}=\cO(1/\log t)$ and $\norm{\bw(t)}=\Theta(\log t)$.

(II) (NGD) let $\bw(t)$ be trained by NGD~\eqref{equ: NGD} with $\eta\leq1$ starting from $\bw(0)=\bzero$. Then $\norm{\frac{\bw(t)}{\norm{\bw(t)}}-\bw^\star}=\cO(1/t)$ and $\norm{\bw(t)}=\Theta(t)$.
\end{lemma}

\begin{theorem}[Theorem 4.4,  \citep{ji2021characterizing}]\label{thm: orthogonal dynamics, ji}
Under the same conditions in Theorem~\ref{thm: NGD upper bound ji}, let $\bw(t)$ be trained by NGD with $\eta\leq1$ starting from $\bw=\bzero$. 
Then 

(i) $\cL_{\perp}(\cdot)$ has a unique minimizer $\bv^\star$ over ${\rm span}\{\cP_{\perp}(\bx_i):i\in\cI\}$; 

(ii) $\cL_{\perp}(\cdot)$ is strongly convex in any bounded set; 

(iii) there exists an absolute constant $C>0$ such that $\norm{\cP_{\perp}(\bw(t))-\bv^\star}\leq C,\ \forall t$.
    
\end{theorem}

\begin{theorem}[Theorem 4,~\citep{soudry2018implicit}]\label{thm: refined error result, soudry}
Under Assumption~\ref{ass: linearly separable} and~\ref{ass: non-degenerate data}, let $\bw(t)$ be trained by GD~\eqref{equ: GD} with $\eta\leq1$ starting from $\bw(0)=\bzero$. 
If we denote $\brho(t)=\bw(t)-\bw^\star\log t$, then 
\begin{align*}
    \lim_{t\to+\infty}\brho(t)=\tilde{\bw},
\end{align*}
where $\tilde{\bw}$ is the solution to the equations: $\eta\exp\left(-\<\tilde{\bw},\bx_i y_i\>\right)=\alpha_i,i\in\cI.$
    
\end{theorem}

\vspace{1.cm}
\section{Useful Inequalities}

\begin{lemma}\label{lemma: basic inequ: sqrt(1+x)}
(i) For any $x\geq0$, $\sqrt{1+x}\leq1+\frac{x}{2}$;
(ii) For any $0\leq x\leq 1/3$, $\sqrt{1+x}\geq1+\frac{x}{3}$.
    
\end{lemma}

\begin{lemma}\label{lemma: increase 3 data NGD}
    For a fixed $\gamma\in(0,1)$, consider the function $h(x)=x+2(1-\gamma^2)\left(\frac{2}{1+e^x}-1\right),x\in\bbR$. Then $h'(x)>0$ holds for any $x\in\bbR$.
\end{lemma}

\begin{proof}[Proof of Lemma~\ref{lemma: increase 3 data NGD}]
\begin{align*}
    h'(x)=1-\frac{4(1-\gamma^2)e^x}{(1+e^x)^2}\geq1-\frac{4(1-\gamma^2)e^x}{(2e^{x/2})^2}=\gamma^2>0,\ \forall x\in\bbR.
\end{align*}
\end{proof}



\begin{lemma}[\cite{bubeck2015convex}]
\label{lemma: convex optimization}
Let the function $f(\cdot):\cS\to\bbR$ be both $L$-smooth and $\mu$-strongly convex on some convex set $\cS\subset\bbR^d$ (in terms of $\ell_2$ norm).
If the minimizer $\bw^\star\in{\rm int}(\cS)$ with $f^\star=f(\bw^\star)$, then for any $\bw,\bw_1,\bw_2\in\cS$, it holds that:
\begin{align*}
    &f(\bw_1)\leq f(\bw_2)+\<\nabla f(\bw_2),\bw_1-\bw_2\>+\frac{L}{2}\norm{\bw_1-\bw_2}^2;
    \\
     &f(\bw_1)\geq f(\bw_2)+\<\nabla f(\bw_2),\bw_1-\bw_2\>+\frac{\mu}{2}\norm{\bw_1-\bw_2}^2;
    \\&
    f(\bw)-f^\star\geq\frac{\mu}{2}\norm{\bw-\bw^\star};
    \\&
    \norm{\nabla f(\bw)}^2\geq2\mu(f(\bw)-f^\star).
\end{align*}
\end{lemma}


\vspace{1.cm}
\section{Experimental Details}\label{appendix: experiment}


\subsection{Experimental details on two synthetic datasets}
\label{appendix: experimental details synthetic}

\begin{itemize}
    \item \underline{Dataset I}. We set $\gamma^\star=\sin(\pi/100)$ and $n=100$. Then we generate the dataset by setting $\bx_1=(\gamma^\star,\sqrt{1-{\gamma^\star}^2})$, $\bx_2=(-\gamma^\star,\sqrt{1-{\gamma^\star}^2})$, and generate $\bx_i\sim{\rm Unif}\left(\bbS^1\cap\{\bx:|x_1|\geq\gamma^\star\}\right)$ randomly for $i\geq3$. As for the label, we set $y_i=\sgn(x_1)$.
    \item \underline{Dataset II}. We set $\gamma^\star=\sin(\pi/100)$ and $n=100$. Then we generate the dataset by setting $\bx_1=(\gamma^\star,\sqrt{1-{\gamma^\star}^2})$, $\bx_2=(-\gamma^\star,\sqrt{1-{\gamma^\star}^2})$, and generate $\bx_i\sim{\rm Unif}\left(\bbB(0,1)\cap\{\bx:|x_1|\geq\gamma^\star\}\right)$ randomly for $i\geq3$. As for the label, we also set $y_i=\sgn(x_1)$.
    \item \underline{PRGD}. We follow the guidelines provided in Theorem~\ref{thm: PRGD main thm}. 
    For the Warm-up phase, we use GD as the Warm-up Phase for 1000 iterations, and then turn it to PRGD. 
    For the second Phase, we employ PRGD(exp) with hyperparameters $T_{k+1}-T_k\equiv5$, $R_k=R_0\times 1.2^k$. 
    To illustrate the role of the progressive radius, we also examine PRGD(poly) configured with $T_{k+1}-T_k\equiv5$, $R_k=R_0\times k^{1.2}$, where the progressive radius increases polynomially. 
\end{itemize}

The numerical results and comparison of different algorithms are shown in Figure~\ref{fig: synthetic data (detailed)} and Table~\ref{tab: comparison time}.

\begin{figure}[!hb]
    \centering
    \subfloat[Dataset I]{
    \includegraphics[width=0.24\textwidth]{figures/out_data.pdf}
    \includegraphics[width=0.22\textwidth]{figures/out_rate_1.pdf}
    \includegraphics[width=0.22\textwidth]{figures/out_rate_2.pdf}
    }
    %
    \\
    \subfloat[Dataset II]{
     \includegraphics[width=0.24\textwidth]{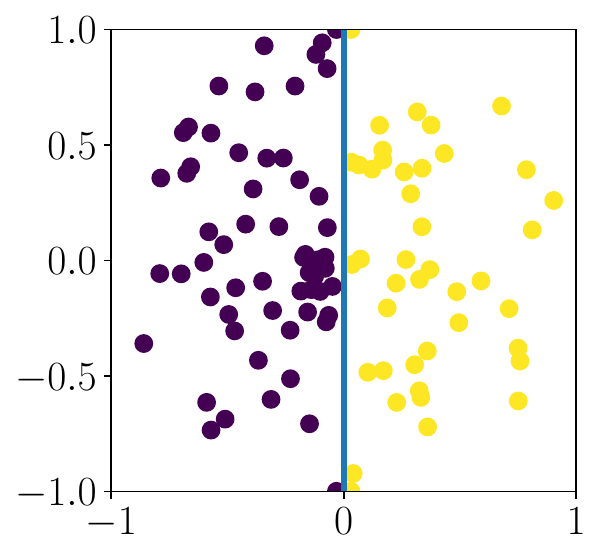}
    \includegraphics[width=0.22\textwidth]{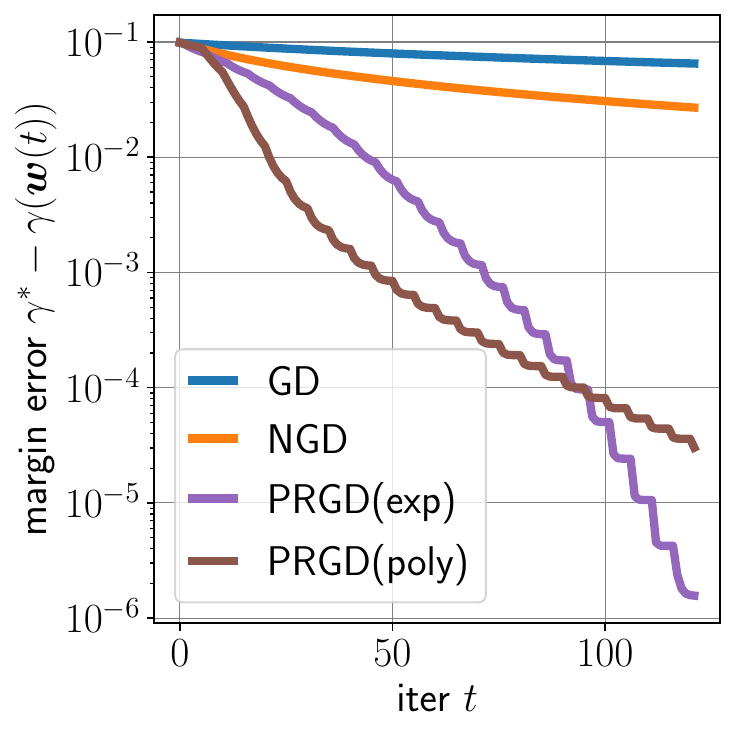}
    \includegraphics[width=0.22\textwidth]{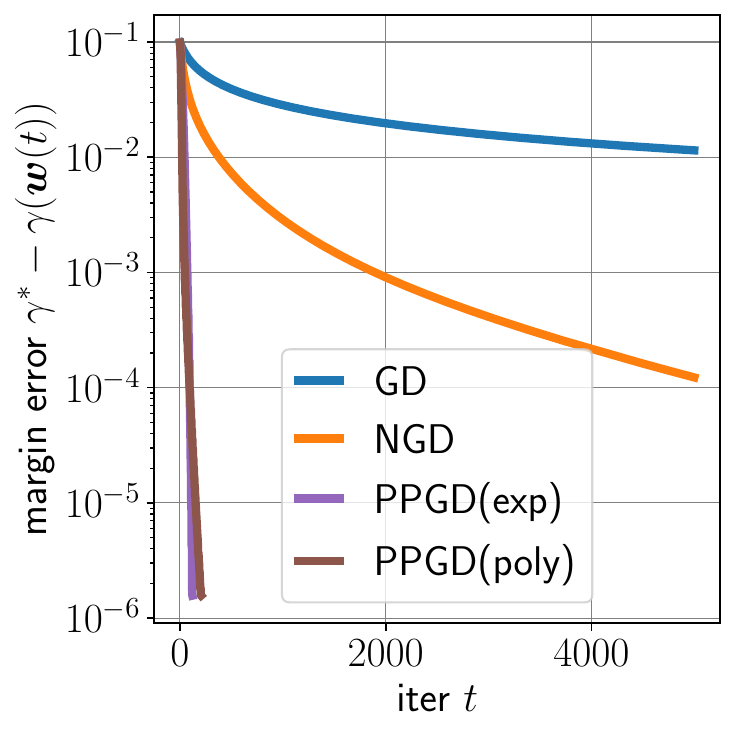}
    }
    \caption{\small (The detailed version of Figure~\ref{fig: synthetic data}) Comparison of margin Maximization rates of different algorithms on two synthetic datasets. (left) The visualization of two 2d synthetic dataset. The yellow points represent the data with label 1, while the purple points corresponds to the data with label 1; 
    (middle)(right) The comparison of margin maximization rates of different algorithms on the corresponding dataset at small and large time scales, respective}
    \label{fig: synthetic data (detailed)}
\end{figure}

\begin{table}[!ht]
     \caption{The number of iterations needed to achieve the same margin error on the sythetic datasets.}
    \centering
    \begin{tabular}{c|c|c|c|c}
    \hline\hline
    & GD & NGD & {\bf PRGD(exp)} & PRGD(poly) \\ \hline
    margin error \texttt{1e-6}, Dataset I  & $+\infty$ & 12,508 & {\bf 106} & 142 \\ \hline
    margin error \texttt{1e-4}, Dataset II & $+\infty$ & 5,027 & {\bf 94} & 95 \\ 
    \hline\hline
    \end{tabular}
    \label{tab: comparison time}
\end{table}

\newpage
\subsection{Experiments Details for VGG on CIFAR-10}
\label{appendix: experimental details VGG}


Following~\citet{lyu2019gradient}, we examine our algorithm for the homogenized VGG-16.
We explored the performance of the VGG-based neural network for image classification tasks on CIFAR-10. 

Our experimental setup involved a modified VGG architecture implemented in PyTorch. 
The homogeneity requires that the bis term exists at most in the first layer~\citep{lyu2019gradient}.
Specifically, the network's architecture comprised multiple convolutional layers without the bias terms, followed by ReLU activations and max pooling. 
The classifier section consisted of three fully connected layers with ReLU activations and dropout, excluding bias in linear transformations.

The network was trained using a batch size of 64, with the option to enable CUDA for GPU acceleration. 
Weight initialization was conducted using Kaiming normalization for convolutional layers and a uniform distribution for linear layers. The model was trained and evaluated using a custom DataLoader for both the training and test datasets.
We used a base learning rate of \(1 \times 10^{-3}\), a momentum of 0.9, and a weight decay of \(5 \times 10^{-4}\). 
For the loss-based learning rate used in NGD and PRGD, we use the strategy in~\citet{lyu2019gradient}.
Additionally, for $\norm{\btheta}$ in the PRGD regime, we use the $\ell_2$ norm of the parameters of all layers. And we configured PRGD with \(T_{k} = 3000 \times \left(2 + \frac{k}{3000}\right)^{3}\), \(R_{k} = \min\left(R_0 \times \left(2 + \frac{k}{3000}\right)^{0.2}, 1000\right)\).


\end{document}